\documentclass[aap]{stylefile}
\RequirePackage{amsthm,amsmath,amsfonts,amssymb}
\RequirePackage[sort,numbers]{natbib}
\RequirePackage{mathtools}
\RequirePackage[shortlabels]{enumitem}
\RequirePackage{algorithm}
\RequirePackage{algorithmic}
\RequirePackage{tikz}
\RequirePackage[colorlinks,citecolor=blue,urlcolor=blue]{hyperref}
\RequirePackage{graphicx}

\DeclareMathOperator*{\argmax}{arg\,max}

\newcommand{\cmark}{\ding{51}}
\newcommand{\xmark}{\ding{55}}

\newcommand\blfootnote[1]{
  \begingroup
  \renewcommand\thefootnote{}\footnote{#1}
  \addtocounter{footnote}{-1}
  \endgroup
}

\usepackage{ragged2e}
\usepackage{bm}
\usepackage{multicol}
\usepackage{multirow}
\usepackage{pifont}

\startlocaldefs
\theoremstyle{plain}
\newtheorem{theorem}{Theorem}[section]
\newtheorem{lemma}{Lemma}[section]
\newtheorem{corollary}{Corollary}[theorem]
\newtheorem{proposition}{Proposition}[section]

\theoremstyle{definition}

\newtheorem{assumption}{Assumption}[section]

\theoremstyle{remark}

\endlocaldefs

\begin{document}

\begin{frontmatter}
\title{Finite-Sample Analysis of Off-Policy Natural Actor-Critic Algorithm
}

\begin{aug}
\author[A]{\fnms{Sajad} \snm{Khodadadian}\ead[label=e2,mark]{zchen458@gatech.edu}},
\author[A]{\fnms{Zaiwei} \snm{Chen}\ead[label=e1,mark]{skhodadadian3@gatech.edu}},
\and
\author[A]{\fnms{Siva Theja} \snm{Maguluri}\ead[label=e3,mark]{siva.theja@gatech.edu}}

\address[A]{
Geogia Institute of Technology,
\printead{e1,e2,e3}}

\blfootnote{Equal contribution between Zaiwei Chen and Sajad Khodadadian}

\end{aug}

\begin{abstract}
In this paper, we provide finite-sample convergence guarantees for an off-policy variant of the natural actor-critic (NAC) algorithm based on Importance Sampling. In particular, we show that the algorithm converges to a global optimal policy with a sample complexity of $\mathcal{O}(\epsilon^{-3}\log^2(1/\epsilon))$ under an appropriate choice of stepsizes. In order to overcome the issue of large variance due to Importance Sampling, we propose the $Q$-trace algorithm for the critic, which is inspired by the V-trace algorithm \cite{espeholt2018impala}. This enables us to explicitly control the bias and variance, and characterize the trade-off between them. 
As an advantage of off-policy sampling, a major feature of our result is that we do not need any additional assumptions, beyond the ergodicity of the Markov chain induced by the behavior policy. 
\end{abstract}

\end{frontmatter}

\section{Introduction}\label{sec:intro}

Reinforcement Learning (RL) is a paradigm where an agent aims at maximizing its cumulative reward by searching for an optimal policy, in an environment modeled as a Markov Decision Process (MDP) \cite{sutton2018reinforcement}. RL algorithms have achieved tremendous successes in a wide range of applications such as self-driving cars with Deep Deterministic Policy Gradient (DDPG) \cite{lillicrap2015continuous}, and AlphaGo in the game of Go \cite{silver2016mastering}. The algorithms in RL can be categorized into value space methods, such as $Q$-learning \cite{watkins1992q}, TD-learning \cite{sutton1988learning}, and policy space methods, such as actor-critic (AC) \cite{konda2000actor}. Despite great empirical successes \cite{wang2016sample, bahdanau2016actor}, the finite-sample convergence of AC type of algorithms are not completely characterized theoretically.

An AC algorithm can be thought as a generalized policy iteration \cite{puterman1995markov}, and consists of two phases, namely actor and critic. The objective  of the actor is to improve the policy, while the critic aims at evaluating the performance of a specific policy. A step of the actor can be thought as a step of Stochastic Gradient Ascent \cite{bottou2018optimization} with preconditioning. An identity pre-conditioner corresponds to regular AC, while a pre-conditioning with fisher information results in natural actor-critic (NAC) \cite{peters2008natural}. As for the critic, to perform a policy evaluation step, it usually uses the TD-learning method and its variants, such as TD$(0)$, or more generally, $n$-step TD \cite{sutton1988learning}. Moreover, such learning process can be done in an on or off-policy manner \cite{degris2012off}.

\textit{Off-policy Actor-Critic.} 
In on-policy AC, the data samples are generated in an online manner, always sampling based on the current policy at hand. In contrast, in this paper, we focus on the off-policy AC, where the algorithm updates the policy based on the data collected (possibly in the past) by a fixed policy, called the \textit{behavior policy}. Off-policy learning is inevitable in high-stakes applications such as healthcare \cite{dann2019policy}, education \cite{mandel2014offline}, robotics \cite{gu2017deep} and clinical trials \cite{liu2018representation, gottesman2020interpretable}. The agent there may not have direct access to the environment in order to perform online sampling, and one has to work with limited historical data that is collected under a fixed behavior policy. Moreover, off-policy AC enables off-line learning by decoupling data collection from learning, and is observed to extract the maximum possible utility out of limited available data  \cite{levine2020offline}.

To account for the difference between the behavior policy and the target policy \cite{geweke1989bayesian} in off-policy algorithms, a popular approach is to use Importance Sampling (IS). The IS ratio, however, can be large in some cases, which might result in high variance \cite{glynn1989importance,precup2000eligibility}. This phenomenon will be illustrated in detail in Section \ref{subsec:Q-trace}. In order to avoid such high variance, one idea is to truncate the IS ratio \cite{ionides2008truncated}, which leads to off-policy TD-learning algorithms such as Retrace$(\lambda)$ \cite{munos2016safe} and V-trace \cite{espeholt2018impala}.

\begin{table*}[t]\label{table: results}
\centering
\caption{Summary of the results in the literature $^1$}
\renewcommand{\arraystretch}{1.25} 
\begin{tabular}{ |c|c|c|c|c|c|  }
 \hline
 \multirow{2}{4.5em}{\centering Algorithm} & \multirow{2}{6em}{\centering Reference} & \multirow{2}{5.5em}{\centering Sample \\
Complexity $^2$} & \multirow{2}{4em}{\centering Single\\trajectory} & \multirow{2}{20 em}{\centering Comments}  \\
  &  &  &  &   \\
 \hline
 \multirow{3}{4.5em}{\centering AC}  &  \cite{wang2019neural} &$\Tilde{\mathcal{O}}(\epsilon^{-6})$ &  \xmark  &  \multirow{3}{20 em}{\centering Function Approx: 
 Sample complexity to ensure $\mathbb{E}[\|\nabla V^{\pi_t}\|^2] \leq \epsilon + \mathcal{E}_{\text{bias}}$ 
 } \\ \cline{2-4}
 & \cite{qiu2019finite} &$\Tilde{\mathcal{O}}(\epsilon^{-4})$&  \xmark & \\\cline{2-4}
 &\cite{kumar2019sample}&$\Tilde{\mathcal{O}}(\epsilon^{-4})$& \xmark &\\
\cline{1-5}
\multirow{3}{4.5em}{\centering NAC}  & \cite{wang2019neural} &$\Tilde{\mathcal{O}}(\epsilon^{-14})$ &   \xmark &  \multirow{2}{20 em}{\centering Function Approx: Sample complexity to ensure $V^{\pi^*}- V^{\pi_t} \leq \epsilon+\mathcal{E}_{\text{bias}}$}\\
\cline{2-4}
  & \cite{agarwal2019theory}    &$\Tilde{\mathcal{O}}(\epsilon^{-6})$
  & \xmark &  \multirow{2}{10 em}{}\\
\cline{2-5}
  & \cite{khodadadian2021finite}    &$\Tilde{\mathcal{O}}(\epsilon^{-4})$
  & \cmark &  \multirow{3}{15 em}{Tabular RL: Convergence to global optimum $V^{\pi^*}- V^{\pi_t} \leq \epsilon$ }\\ 
\cline{1-4}
\multirow{2}{4.5 em}{\centering Off-Policy\\ NAC}
  & \multirow{2}{6 em}{\centering Our work}   &   \multirow{2}{5.5 em}{\centering $\Tilde{\mathcal{O}}(\epsilon^{-3})$} &  \multirow{2}{4 em}{\centering \cmark}  & \\ 
&     &   &    & \\
 \hline
\end{tabular}

\justify
{\small $^{1}$ 
There are two other related works \cite{xu2020improving} and \cite{xu2020non}. \cite{xu2020improving} claims a sample complexity of $\Tilde{\mathcal{O}}(\epsilon^{-2})$ for NAC. \cite{xu2020non} claims a sample complexity of $\Tilde{\mathcal{O}}(\epsilon^{-2.5})$ for AC and $\Tilde{\mathcal{O}}(\epsilon^{-4})$ for NAC. In our opinion, the interpretation of the convergence results in terms of sample complexity in both papers is incorrect.  In case one accepts the interpretation in \cite{xu2020improving,xu2020non}, our results imply a sample complexity of  $\Tilde{\mathcal{O}}(\epsilon^{-1/N})$ for \textit{an arbitrary} $N\in\mathbb{Z}^+$. See Appendix \ref{ap:literature:convergence_bound} for a detailed explanation.\\
$^2$ In this table, $\Tilde{O}(\cdot)$ ignores all the logarithmic terms. See Appendix \ref{ap:compute-sample-complexity} for detailed calculations and comments regarding the sample complexities presented here.}

\end{table*}

\subsection{Main Contributions}\label{subsec:contribution}
In this paper, we study finite-sample convergence guarantees of an off-policy variant of the NAC algorithm. Our main contributions are threefold.

\textit{$Q$-Trace for Off-Policy TD-Learning: Algorithm and Finite-Sample Bounds.} To estimate the $Q$-function for the critic, we propose an off-policy TD-learning algorithm called $Q$-trace. This is inspired by the V-trace algorithm \cite{espeholt2018impala} to estimate the $V$-function. We establish the finite-sample convergence bounds of $Q$-trace, and show how the truncated IS ratios can be used to explicitly trade-off the truncation bias and the variance.

\textit{Finite-Sample Bounds for Off-Policy NAC.} 
Based on the $Q$-trace algorithm for the critic, we propose an off-policy NAC algorithm, which uses only a \textit{single trajectory} of samples.
To the best of our knowledge, we establish the first known finite-sample convergence guarantees of an \textit{off-policy} NAC algorithm. Based on that, we show that in order to obtain an $\epsilon$-optimal policy, the amount of samples required is of the size $\mathcal{O}(\epsilon^{-3}\log^2(1/\epsilon))$. 
This shows that the off-policy NAC outperforms even the best known theoretical convergence bounds of \textit{on-policy} NAC algorithms. See Table \ref{table: results} for more details.  

\textit{Exploration through Off-Policy Sampling.} 
While off-policy learning is primarily motivated by practical constraints, in this paper, we demonstrate that off-policy sampling leads to natural exploration. By exploiting off-policy sampling, we do not require either hard-to-verify assumptions made in the literature to ensure exploration \cite{xu2020improving, wu2020finite}, or additional exploration steps in the algorithm that slow down the convergence \cite{khodadadian2021finite}.

\subsection{Related Work}\label{subsec:literature}
Two popular algorithms for finding the optimal policy of an MDP are value iteration and policy iteration, which corresponds to $Q$-learning and AC in Reinforcement Learning when the underlying model is unknown.

\textit{The $Q$-learning algorithm}, first proposed in \cite{watkins1992q} is one of the most celebrated value space methods for solving the RL problem \cite{sutton2018reinforcement}. Since the proposal, there has been a long line of work to establish the convergence properties of $Q$-learning. In particular, \cite{tsitsiklis1994asynchronous,jaakkola1994convergence,bertsekas1996neuro,borkar2000ode, borkar2009stochastic} characterize the asymptotic convergence of $Q$-learning, \cite{beck2012error, beck2013improved, wainwright2019stochastic, chen2020finite,chen2021finite} study the finite-sample convergence bound in the mean-square sense, and \cite{even2003learning, li2020sample, qu2020finite} study the high-probability convergence bounds.

In AC framework, usually the actor uses Policy Gradient (PG) to perform policy update, and the critic uses TD-learning method to perform policy evaluation.

\textit{The PG method} was shown to converge in \cite{sutton1999policy, baxter2001infinite, pirotta2015policy, haarnoja2017reinforcement}. Natural PG, which is a PG method with preconditioning, was proposed in \cite{kakade2001natural}. More recently, there has been a line of work to establish finite-sample convergence bound of (natural) PG algorithm \cite{even2009online, azar2012dynamic, geist2019theory, agarwal2019theory, wang2019neural, liu2019neural, shani2020adaptive, mei2020global, cen2020fast, bhandari2020note}.

\textit{TD-learning method}, originally proposed in \cite{sutton1988learning}, represents a family of policy evaluation algorithms in RL. The asymptotic convergence of TD-learning has been established in \cite{tsitsiklis1994asynchronous,jaakkola1994convergence,borkar2000ode}. Furthermore, the finite-sample convergence bounds of TD-learning have been studied in \cite{dalal2018finite, lakshminarayanan2018linear, bhandari2018finite, srikant2019finite} in the on-policy setting. Off-policy variants of TD-learning such as Retrace$(\lambda)$, Tree-backup, and V-trace were studied in \cite{munos2016safe,precup2000eligibility,espeholt2018impala} respectively. Finite-sample bounds for V-trace are quantified in \cite{chen2020finite,chen2021finite}. 

\textit{Actor-critic}, as a stochastic variant of policy iteration, was proposed in \cite{barto1983neuronlike, borkar1997actor}, and later it has extended to function approximation setting \cite{konda2000actor} and NAC \cite{peters2008natural, NIPS2009_3767, NIPS2013_5184, bhatnagar2009natural}. Asymptotic convergence of AC algorithms
was studied in \cite{williams1990mathematical, konda2000actor, borkar1997actor, borkar2009stochastic, maei2018convergent, zhang2019convergence, zhang2020provably}. Furthermore, there has been a flurry of recent work studying the finite-sample convergence of AC and NAC \cite{qiu2019finite, kumar2019sample, shani2020adaptive, wang2019neural, xu2020non, xu2020improving, wu2020finite, khodadadian2021finite}. The results are summarized in Table \ref{table: results}. Concurrent work \cite{lan2021policy} studies a variant of NAC with on-policy sampling and time-varying inverse temperature, and obtains an $\Tilde{\mathcal{O}}(\epsilon^{-2})$ sample complexity.

The rest of this paper is organized as follows. In Section \ref{sec:RL}, we first present the $Q$-trace algorithm for off-policy TD-learning. We then use it with the Natural Policy Gradient to get the off-policy NAC algorithm, and present the finite-sample convergence bounds and sample complexity analysis. In Section \ref{sec:proof}, we present the proof sketch of our main results, and conclude in Section \ref{sec:con}.

\section{Off-Policy Natural Actor-Critic: Algorithm and Finite-Sample Bounds}\label{sec:RL}
\subsection{Background on Reinforcement Learning}
We model our RL problem with an MDP which consists of a tuple of 5 elements $(\mathcal{S},\mathcal{A},\mathcal{R}, \mathcal{P},\gamma)$. Here $\mathcal{S}$ and $\mathcal{A}$ are finite sets of states and actions, $\mathcal{R}: \mathcal{S}\times\mathcal{A}\rightarrow [0,1]$ is the reward function, $\mathcal{P}: \mathcal{S}\times\mathcal{A}\rightarrow \Delta^{|\mathcal{S}|}$ (where $\Delta^{|\mathcal{S}|}$ is the probability simplex on $\mathbb{R}^{|\mathcal{S}|}$) is the collection of transition probabilities that are unknown, and $\gamma\in(0,1)$ is the discount factor. 

The dynamics of an MDP is as follows. At each time step $k$, the system is at some state $S_k$ of the environment. The agent chooses an action $A_k$ based on a policy $\pi$ at hand, $A_k\sim \pi(\cdot|S_k)$, and the system moves to a new state based on the transition probabilities $\mathbb{P}(S_{k+1}=\cdot|S_k,A_k)$, and induces a one-step reward $\mathcal{R}(S_k,A_k)$. The goal of the agent to find an optimal policy which maximizes the cumulative reward. Specifically, the value function of a policy $\pi$ is defined by $V^\pi(\mu) = \mathbb{E}[\sum_{k=0}^\infty \gamma^k \mathcal{R}(S_k,A_k)|S_0\sim \mu, A_k\sim\pi(\cdot|S_k)]$, where $\mu$ is an initial distribution over states. Then the goal is to find an optimal policy $\pi^*$ s.t.
\begin{align}\label{eq:optimization_problem}
    \pi^*\in \argmax_{\pi\in \Pi} V^{\pi}(\mu),
\end{align}
where $\Pi$ represents the set of all policies. 

\subsection{Natural Policy Gradient}

Policy gradient algorithms aim at solving the optimization problem \eqref{eq:optimization_problem} by using gradient ascent or its variants in the policy space. In particular, a Mirror Descent (MD) \cite{nemirovskij1983problem} update of policy with stepsize $\beta$ reads as:
\begin{align}
    \pi_{t+1}=\argmax_{\pi\in\Pi} \left\{\beta \langle\nabla V^{\pi_t}(\mu),\pi-\pi_t\rangle-B(\pi,\pi_t)\right\}\label{eq:MD_update},
\end{align}
where $B(\cdot,\cdot)$ is an appropriately chosen Bregman divergence between two policies. If we replace the Bregman divergence with 
$B(\pi,\pi_t)=\sum_s d_\mu^{\pi_t}(s)\mathcal{KL}(\pi(\cdot|s)|\pi_t(\cdot|s))$ in Eq. (\ref{eq:MD_update}), we get the Natural Policy Gradient (NPG) algorithm for MDPs. Here $d_\mu^{\pi}(s) = (1-\gamma)\sum_{j=0}^\infty \gamma^j\mathbb{P}^\pi(S_j=s\mid S_0\sim \mu)$ is the discounted state visitation distribution \cite{agarwal2019theory}, and $\mathcal{KL}(\cdot\mid\cdot)$ stands for the KL-Divergence \cite{cover1999elements}. It has been shown in \cite{agarwal2019theory} that the update equation \eqref{eq:MD_update} can be equivalently written as
\begin{align} \label{eq:exp_weight_update}
    \pi_{t+1}(a|s) = \frac{\pi_t(a|s)\exp(\beta Q^{\pi_t}(s,a))}{\sum_{a'}\pi_t(a'|s)\exp(\beta Q^{\pi_t}(s,a'))}, \forall\; s,a,
\end{align}
where $Q^{\pi}(s,a) = \mathbb{E}_\pi[\sum_{k=0}^\infty \gamma^k \mathcal{R}(S_k,A_k)|S_0=s,A_0=a]$ is the $Q$-function for policy $\pi$ \cite{puterman1995markov}. The update rule \eqref{eq:exp_weight_update} can be equivalently derived using the preconditioned gradient ascent (with the Moore–Penrose inverse of the Fisher information matrix as the pre-conditioner) on the dual space of the policy $\pi$. This interpretation of \eqref{eq:exp_weight_update} was presented in \cite{kakade2001natural,agarwal2019theory}.
Furthermore, an interpretation of \eqref{eq:exp_weight_update} in terms of Mirror Descent Modified Policy Iteration (MD-MPI) was presented in \cite{geist2019theory}. An important result about the NPG is that, although the objective function of \eqref{eq:optimization_problem} is not concave, it has been shown in \cite{agarwal2019theory} that the policies achieved by the MD update of \eqref{eq:exp_weight_update} converges to an optimal policy with rate $\mathcal{O}(1/t)$.

\begin{algorithm*}[t]\caption{$Q$-Trace}\label{alg:Q-trace}
\begin{algorithmic}[1] 
	\STATE {\bfseries Input:} $K$, $\alpha$, $Q_0$, $\pi$, $\bar{\rho}$, and $\bar{c}$, $\{(S_k,A_k)\}_{0\leq k\leq K+n}$ (generated by the behavior policy $\pi_b$)
	\FOR{$k=0,1,\cdots,K-1$}
	\STATE $\alpha_k(s,a)=\alpha\mathbb{I}_{\{(s,a)=(S_k,A_k)\}}$ for all $(s,a)$
	\STATE $\Delta_{k,i}=\mathcal{R}(S_i,A_i)+\gamma\rho_\pi(S_{i+1},A_{i+1})Q_k(S_{i+1},A_{i+1})-Q_k(S_i,A_i)$ for all $k\leq i\leq k+n-1$ \label{lst:line:blah2}
	\STATE $Q_{k+1}(s,a)=Q_k(s,a)+\alpha_k(s,a)\sum_{i=k}^{k+n-1}\gamma^{i-k}\prod_{j=k+1}^ic_\pi(S_j,A_j)\Delta_{k,i}$ for all $(s,a)$
	\ENDFOR
	\STATE\textbf{Output:} $Q_K$
\end{algorithmic}
\end{algorithm*}

Although the convergence result in  \cite{agarwal2019theory} is promising to find the optimal policy in an MDP, since we do not have access to the transition probabilities and so the $Q$-function in RL, we cannot update the policy according to Eq. \eqref{eq:exp_weight_update}. Natural Actor-Critic (NAC) algorithm, which is a sample-based variant of the update \eqref{eq:exp_weight_update}, proceeds as follows. In each iteration, first the critic generates an estimate $Q_{t}$ of the $Q$-function $Q^{\pi_t}$. Then the actor updates the policy according to Eq. \eqref{eq:exp_weight_update} with $Q^{\pi_t}$ replaced by the estimate $Q_t$. 

\subsection{The Q-Trace Algorithm for Off-Policy Prediction}\label{subsec:Q-trace}
In this section, we focus on the critic sub-problem, and develop the $Q$-trace algorithm to estimate $Q^{\pi_t}$. $Q$-trace is an off-policy variant of TD-learning based on Importance Sampling. Crucially, we introduce two different truncation levels for the IS ratios in order to explicitly control the trade-off between truncation bias and the variance. This is inspired by the V-trace algorithm in \cite{espeholt2018impala}.

We next introduce our notations to describe the $Q$-trace algorithm. Let $\pi$ be the target policy (i.e., we want to evaluate $Q^{\pi}$) and $\pi_b$ be the behavior policy (i.e., we use $\pi_b$ to collect samples). We assume that the behavior policy $\pi_b$ satisfies $\pi_b(a|s)>0$ for any $(s,a)$. This is typically necessary in off-policy setting. Let $\bar{\rho}$ and $\bar{c}$ be two truncation levels satisfying $\bar{\rho}\geq \bar{c}\geq 1$. Define $c_\pi(s,a)=\min(\bar{c},\frac{\pi(a|s)}{\pi_b(a|s)})$ and $\rho_\pi(s,a)=\min(\bar{\rho},\frac{\pi(a|s)}{\pi_b(a|s)})$ for all $(s,a)$, which are the truncated IS ratios. 

The off-policy $Q$-trace algorithm is presented in Algorithm \ref{alg:Q-trace}. To better understand Algorithm \ref{alg:Q-trace}, consider the following special cases. Suppose we use on-policy sampling, that is, $\pi_b=\pi$. Set $\bar{\rho}=\bar{c}=1$. Observe that in this case we have $c_\pi(s,a)=\rho_\pi(s,a)=1$ for all $(s,a)$. Then Algorithm \ref{alg:Q-trace} reduces to the regular $n$-step TD, which is known to converge to $Q^\pi$ \cite{tsitsiklis1994asynchronous, sutton2018reinforcement}. 

In the off-policy setting (i.e., $\pi_b\neq \pi$), suppose we choose $\bar{\rho}=\bar{c}\geq \max_{(s,a)}\frac{\pi(a|s)}{\pi_b(a|s)}$. Then we have $c_\pi(s,a)=\rho_\pi(s,a)=\frac{\pi(a|s)}{\pi_b(a|s)}$, hence there is essentially no truncation. In this case, Algorithm \ref{alg:Q-trace} corresponds to the standard $n$-step TD using off-policy sampling, and therefore converges to $Q^\pi$ \cite{precup2000eligibility}.

A fundamental problem in off-policy TD is that the variance in the estimate can be very large or even infinity \cite{glynn1989importance,munos2016safe}. This is mainly because of the product of IS ratios $\prod_{j=k+1}^{i}\frac{\pi(A_j|S_j)}{\pi_b(A_j|S_j)}$.
To have control on the variance of the estimate, we introduce the truncation levels $\bar{\rho}$ and $\bar{c}$. However, due to the truncation, the IS ratios are now biased, and hence the algorithm no longer converges to the target value function $Q^\pi$. In fact, Algorithm \ref{alg:Q-trace} converges to a biased limit point, denoted by $Q^{\bar{\rho},\pi}$, which need not necessarily  be the value function of any policy.

Importantly, the limit point $Q^{\bar{\rho},\pi}$ depends only on the target policy $\pi$ and the truncation level $\bar{\rho}$, but not on the truncation level $\bar{c}$. Therefore, we can heavily truncate the IS ratio $c_\pi(s,a)$ by using small $\bar{c}$ without affecting the limit point of the $Q$-trace algorithm. In fact, as we will see in Section \ref{subsec:results}, this is exactly what we should do.
To quantify the truncation bias of Algorithm \ref{alg:Q-trace}, we have the following result.
\begin{lemma}\label{le:bias}
For any $\bar{\rho}\geq 1$ and policy $\pi$, we have (1) $\|Q^{\bar{\rho},\pi}-Q^\pi\|_\infty\leq \frac{\max_{(s,a)}\max(\pi(a|s)-\bar{\rho}\pi_b(a|s),0)}{(1-\gamma)^2}$, and (2) $\|Q^{\bar{\rho},\pi}\|_\infty\leq \frac{1}{1-\gamma}$.
\end{lemma}
Observe from Lemma \ref{le:bias} (1) that when $\bar{\rho}\geq \max_{s,a}\frac{\pi(a|s)}{\pi_b(a|s)}$, we have $Q^{\bar{\rho},\pi}=Q^\pi$. This makes intuitive sense in that when $\bar{\rho}$ is large, there is essentially no truncation in the IS ratio $\rho_\pi(s,a)$, and we should not expect any truncation bias. 

\textit{Comparison to Related Algorithms.} There are two algorithms in the literature that are closely related to our $Q$-trace algorithm, namely the Retrace$(\lambda)$ in \cite{munos2016safe} and the V-trace in \cite{espeholt2018impala}. The Retrace$(\lambda)$ algorithm in \cite{munos2016safe} is proposed to evaluate the $Q$-function, but uses a single truncation level. 
In contrast, we have two truncation levels $\bar{c}$ and $\bar{\rho}$, which enables us to trade-off the truncation bias and variance. 

V-trace, an off-policy variant of TD to estimate the $V$-function, first introduced the idea of using two truncation levels. However, there are several differences between $Q$-trace and V-trace. First, the product of the IS ratios $c_\pi(S_j,A_j)$ starts from $j=k+1$ rather than $j=k$ in V-trace. This simple but important modification enables us to get a convergence bound in Theorem \ref{thm:Q-trace} which does not dependent on the target policy $\pi$. This is essential for us to use the $Q$-trace algorithm in the AC framework, as after each iteration of the actor, the critic receives a different policy $\pi_t$ to evaluate. Second, as opposed to V-trace, where the limit point is a value function of some policy, the limit point $Q^{\bar{\rho},\pi}$ of $Q$-trace is not necessarily the $Q$-function of any policy. Finally, due to the structure of the $Q$-function, the IS ratio $\rho_{\pi}(S_{i+1},A_{i+1})$ is multiplied with only one of the three terms in the temporal difference $\Delta_{k,i}$ (Algorithm \ref{alg:Q-trace} line 4), as opposed to all the three terms in V-trace.

In summary, we propose the off-policy $Q$-trace algorithm to evaluate the $Q$-function in the critic. Moreover, the flexibility of choosing the truncation levels in $Q$-trace enables us to explicitly trade-off the truncation bias and the variance.

\subsection{Off-Policy Natural Actor-Critic Algorithm}\label{subsec:NAC}
We are now ready to present our off-policy NAC algorithm \ref{alg:off_policy_NPG}. In iteration $t$, the critic first estimates the $Q$-function $Q^{\pi_t}$ using the $Q$-trace algorithm, which itself runs over $K$ iterations. Then the actor uses the estimate $Q_{t+1}$ in Eq. (\ref{eq:exp_weight_update}) to perform a policy update. Thus, we have a two-loop algorithm. 

\begin{algorithm}[h]\caption{Off-Policy Natural Actor-Critic}\label{alg:off_policy_NPG}
\begin{algorithmic}[1] 
	\STATE {\bfseries Input:} $T$, $K$, $\alpha$, $\beta$, $Q_0=\bm{0}$, $\pi_0$, $\bar{\rho}$, $\bar{c}$, and $\{(S_k,A_k)\}_{0\leq k\leq T(K+n)}$ (a \textit{single trajectory} generated by the behavior policy $\pi_b$)\\
	\FOR{$t=0,1,\cdots,T-1$}
	\STATE \textbf{Critic update:}
	\STATE $\text{DataSet}=\{(S_i,A_i)\}_{t(K+n)\leq i\leq (t+1)(K+n)}$
	\STATE $Q_{t+1}=Q\text{-Trace}(K,\alpha,Q_0,\pi_t,\bar{c},\bar{\rho},\text{DataSet})$
	\STATE \textbf{Actor update:}
	\STATE $\pi_{t+1}(a|s) = \frac{\pi_t(a|s)\exp(\beta Q_{t+1}(s,a))}{\sum_{a'}\pi_t(a'|s)\exp(\beta Q_{t+1}(s,a')) }$ $\forall\;(s,a)$ \label{eq:actor}
	\ENDFOR
	\STATE\textbf{Output:} $\{\pi_t\}_{0\leq t\leq T-1}$
\end{algorithmic}
\end{algorithm}

In Algorithm \ref{alg:off_policy_NPG}, due to off-policy sampling, the sampling process and the learning process are decoupled, which allows the agent to learn in an off-line manner \cite{levine2020offline}. Moreover, note that we are using a \textit{single trajectory} of samples $\{(S_k,A_k)\}_{0\leq k\leq T(K+n)}$ to perform the update. In related literature \cite{xu2020non,xu2020improving,wang2019neural},
sampling needs to be 
often restarted with an arbitrary initial state,  which is not practical in many real-world applications. See Appendix \ref{ap:literature:trajectory} for more details.

\subsection{Finite-Sample Convergence Guarantees}\label{subsec:results}

In this section, we present our main results about the finite-sample convergence bounds of the $Q$-trace algorithm \ref{alg:Q-trace} for off-policy TD-learning, and the off-policy NAC Algorithm \ref{alg:off_policy_NPG}. We begin by stating our one and only assumption. 

\begin{assumption}\label{as:MC}
The Markov chain $\{S_k\}$ induced by $\pi_b$ is irreducible and aperiodic.
\end{assumption}

Assumption \ref{as:MC} is commonly made in related work about RL algorithms with Markovian sampling \cite{tsitsiklis1997analysis, tsitsiklis1999average, maei2018convergent, zhang2020provably}, and it implies that the Markov chain $\{S_k\}$ has a unique stationary distribution $\mu_b\in\Delta^{|\mathcal{S}|}$. Moreover, since the state space $\mathcal{S}$ is finite, there exist $C>0$ and $u\in (0,1)$ such that 
\begin{center}
    $\|P^k(s,\cdot)-\mu_b(\cdot)\|_{\text{TV}}\leq Cu^k$
\end{center}
for any $k\geq 0$ and $s\in\mathcal{S}$, where $\|\cdot\|_{\text{TV}}$ is the total variation distance \cite{levin2017markov}.

A major issue in the design of AC algorithms is to ensure enough exploration to all state-action pairs $(s,a)$.
It was demonstrated in \cite{khodadadian2021finite} that the algorithm can get stuck in a local optimum if there is not enough exploration. Sampling from a fixed policy that leads to an ergodic Markov chain naturally ensures exploration, and so we do not need any additional assumptions. In contrast, prior literature on the analysis of on-policy AC either makes additional assumptions that are hard to satisfy \cite{xu2020improving, wu2020finite} or introduce an additional exploration step in the algorithm \cite{khodadadian2021finite} that slows the convergence. See Appendix \ref{ap:literature:exploration} for more details.

To state our result, we need the following notation. Let $\tau_\alpha=\min\{k\geq 0:\max_{s\in\mathcal{S}}\|P^k(s,\cdot)-\mu_b(\cdot)\|_{\text{TV}}\leq \alpha\}$, where $\alpha$ is the constant stepsize used in the critic step of Algorithm \ref{alg:off_policy_NPG}. The quantity $\tau_\alpha$ can be viewed as the \textit{mixing time} of the Markov chain $\{S_k\}$ with accuracy $\alpha$. Furthermore, under the geometric mixing property (implied by Assumption \ref{as:MC}), the mixing time $\tau_\alpha$ can be bounded by $L(\log(1/\alpha)+1)$ for some constant $L>0$. Let $f(\bar{c},\gamma)=\frac{1-(\gamma\bar{c})^n}{1-\gamma\bar{c}}$ when $\gamma\bar{c}\neq1$, and $=n$ when $\gamma\bar{c}=1$. Suppose the constant stepsize $\alpha$ within the critic is properly chosen. The explicit condition is given in Appendix \ref{pf:thm:Q-trace}. Then we have the following result.

\begin{theorem}\label{thm:Q-trace}
Consider $\{Q_k\}$ of Algorithm \ref{alg:Q-trace}. Suppose that (1) Assumption \ref{as:MC} is satisfied, (2) $Q_0$ is initiated at $\bm{0}$, and (3) the constant stepsize $\alpha$ is chosen such that $\alpha(\tau_\alpha+n+1)\leq \min\left(\frac{1}{12(\bar{\rho}+1)f(\bar{c},\gamma)},\frac{(1-\gamma_c)^2}{8208(\bar{\rho}+1)^2f(\bar{c},\gamma)^2\log(|\mathcal{S}||\mathcal{A}|)}\right)$, where $\gamma_c\in (0,1)$ (defined in Proposition \ref{prop:Q-trace-property} (3) (b)) does not depend on the target policy $\pi$, Then we have for all $k\geq \tau_\alpha+n+1$:
\begin{align*}
    \mathbb{E}[\|Q_k-Q^{\bar{\rho},\pi}\|_\infty^2]\leq\;&\underbrace{\frac{c_1}{(1-\gamma)^2}\left(1-\frac{1-\gamma_c}{2}\alpha\right)^{k- (\tau_\alpha+n+1)}}_{T_1:  \text{Convergence Bias}}\\
    &+\underbrace{\frac{c_2\log(|\mathcal{S}||\mathcal{A}|)}{(1-\gamma_c)^2(1-\gamma)^2}(\bar{\rho}+1)^2f(\bar{c},\gamma)^2\alpha (\tau_\alpha+n+1),}_{T_2: \text{Convergence Variance}}
\end{align*}
    where $c_1$ and $c_2$ are numerical constants.
\end{theorem}
Observe that the RHS of the convergence bound does not depend on the target policy $\pi$. This is important for us to later use Theorem \ref{thm:Q-trace} to show the finite-sample guarantees of off-policy NAC algorithm \ref{alg:off_policy_NPG}.

This result characterizes the rate of convergence of $Q$-trace algorithm to its stationary point, $Q^{\bar{\rho},\pi}$. The error on the RHS has two terms, which are called bias and variance respectively in the SA literature \cite{chen2020finite}. To contrast this with the bias due to truncation, we call it the convergence bias. The second error term is simply called the variance. Theorem \ref{thm:Q-trace} implies that under an appropriate constant stepsize $\alpha$, while the $Q$-trace algorithm achieves exponentially decaying convergence bias, it leads to a constant variance that cannot be eliminated, and is of the size $\mathcal{O}(\alpha \log(1/\alpha))$. The logarithmic factor is due to the mixing time $\tau_\alpha$, which arises as a consequence of performing Markovian sampling of $\{(S_k,A_k)\}$. 

The following corollary provides the error of the estimate $Q_k$ with respect to the true $Q$-function $Q^{\pi}$.

\begin{corollary}\label{co:critic}
    Under the same assumptions of Theorem \ref{thm:Q-trace}, we have for all $k\geq \tau_\alpha+n+1$:\\
    $\mathbb{E}\left[\|Q_k-Q^{\pi}\|_\infty\right]\leq \sqrt{T_1}+\sqrt{T_2}+\frac{\max(1-\bar{\rho}\min_{s,a}\pi_b(a|s),0)}{(1-\gamma)^2}$,\\
    where the terms $T_1$ and $T_2$ are given in Theorem \ref{thm:Q-trace}.
\end{corollary}

The proof of Corollary \ref{co:critic} immediately follows by combining Lemma \ref{le:bias} with Theorem \ref{thm:Q-trace} and using Jensen's inequality. We next present the finite-sample bound of the off-policy NAC algorithm \ref{alg:off_policy_NPG}.

\begin{theorem}\label{thm:main}
Consider $\{\pi_t\}$ generated by Algorithm \ref{alg:off_policy_NPG}. Suppose that Assumption \ref{as:MC} is satisfied, and $K\geq \tau_\alpha+n+1$. Then we have the following performance bound:
\begin{align*}
    V^{\pi^*}(\mu)-\max_{0\leq t\leq T-1}\mathbb{E}\left[V^{\pi_{t}}(\mu)\right]
    \leq \;&\underbrace{\frac{24}{(1-\gamma)^3}\left(1-\frac{1-\gamma_c}{2}\alpha\right)^{\frac{1}{2}(K- (\tau_\alpha+n+1))}}_{E_1:  \text{Convergence bias in the Critic}}\\
    &+\underbrace{\frac{1200\log^{1/2}(|\mathcal{S}||\mathcal{A}|)}{(1-\gamma)^3(1-\gamma_c)}(\bar{\rho}+1)f(\bar{c},\gamma)[\alpha (\tau_\alpha+n+1)]^{1/2}}_{E_2: \text{Variance in the Critic}}\\
    &+\underbrace{\frac{4\max(0,1-\bar{\rho}\min_{s,a}\pi_b(a|s))}{(1-\gamma)^4}}_{E_3: \text{Truncation bias}}\\
    &+\underbrace{\frac{\log(e|\mathcal{A}|)}{(1-\gamma)^2\beta T}}_{E_4: \text{Convergence error in the Actor}},
\end{align*}
\end{theorem}

The terms $E_1$ and $E_2$ correspond to the two terms on the RHS of the convergence bounds in Theorem \ref{thm:Q-trace}, and capture the convergence bias and the variance in the critic estimate. We now focus on the terms $E_3$ and $E_4$, and the trade-off between the variance $E_2$ and the truncation bias $E_3$.

\textit{Error Due to Truncated IS Ratio.} The term $E_3$ accounts for the error due to introducing the truncation level $\bar{\rho}$ in the critic (i.e., the $Q$-trace Algorithm \ref{alg:Q-trace}). Recall that because of $\bar{\rho}$, the limit point of the critic is $Q^{\bar{\rho},\pi_t}$ instead of $Q^{\pi_t}$. Note that when $\bar{\rho}\geq 1/\min_{s,a}\pi_b(a|s)$ (which implies $\bar{\rho}\geq \max_{s,a}\frac{\pi_t(a|s)}{\pi_b(a|s)}$ for any $t$), there is essentially no truncation in the IS ratio $\rho_{\pi_t}(s,a)$, and hence we have $E_3=0$, which agrees with Lemma \ref{le:bias}.

\textit{Error Bound of the Actor.} The term $E_4$ is due to the error in the actor update. That is, $E_4$ would be the only error term we have if we can directly use $Q^{\pi_t}$ in the actor update of Algorithm \ref{alg:off_policy_NPG}. Observe that $E_4=\mathcal{O}(\frac{1}{T})$, which agrees with results in \cite{agarwal2019theory} [Theorem 5.3].

\textit{Bias-Variance Trade-Off.} Recall that the motivation for introducing the truncation levels $\bar{\rho}$ and $\bar{c}$ is to control the variance in the critic estimate. We first consider the impact of $\bar{\rho}$. Observe that the term $E_3$ is in favor of large $\bar{\rho}$ while the term $E_2$ grows linearly with respect to $\bar{\rho}$. Therefore, there is an explicit trade-off between the variance and the truncation bias in choosing $\bar{\rho}$. As a result, if we want to have convergence to the global optimal, by choosing $\bar{\rho}= 1/\min_{s,a}\pi_b(a|s)$, we introduce an additional $1/\min_{s,a}\pi_b(a|s)$ factor in the variance term $E_2$. 

The truncation level $\bar{c}$ appears only in the variance term $E_2$. In view of the expression of $f(\bar{c},\gamma)$ (defined before Theorem \ref{thm:Q-trace}), we should choose $\bar{c}$ such that $\bar{c}\gamma <1$ to avoid an exponential factor in the variance term. These observations are similar to \cite{espeholt2018impala,chen2020finite,chen2021finite}, where the V-trace algorithm is studied.

One drawback with Theorem \ref{thm:main} is that the error bound is stated in terms of $\max_{0\leq t\leq T-1}\mathbb{E}[V^{\pi_t}(\mu)]$, while in practice we do not know which policy among $\{\pi_t\}_{0\leq t \leq T-1}$ has the best performance. To overcome this problem, using standard techniques in optimization \cite{lan2020first}, we can obtain the following refined performance bound of Algorithm \ref{alg:off_policy_NPG}.

\begin{corollary}\label{co:last-iterate}
Let $T'$ be a random sample uniformly drawn from $\{0,1,...,T-1\}$. Then we have the following performance guarantee on $\pi_{T'}$:
\begin{align*}
    V^{\pi^*}(\mu)-\mathbb{E}\left[V^{\pi_{T'}}(\mu)\right]\leq E_1+E_2+E_3+E_4,
\end{align*}
where the terms $\{E_i\}_{1\leq i\leq 4}$ are given in Theorem \ref{thm:main}.
\end{corollary}

The convergence guarantee in in Corollary \ref{co:last-iterate} holds for the policy attained by Algorithm \ref{alg:off_policy_NPG} at a random point between $0$ and $T-1$. However, in practice one usually takes the last policy achieved by the algorithm as the output. Numerical experiments of off-policy NAC algorithm \ref{alg:off_policy_NPG} in Figure \ref{fig:1} shows that in expectation, the algorithm can converges almost monotonically. Theoretically showing a performance bound for $V^{\pi^*}(\mu)-\mathbb{E}\left[V^{\pi_{T-1}}(\mu)\right]$ is a future direction of this work. 
\begin{figure}[ht]
    \centering
    \includegraphics[width=0.4\linewidth]{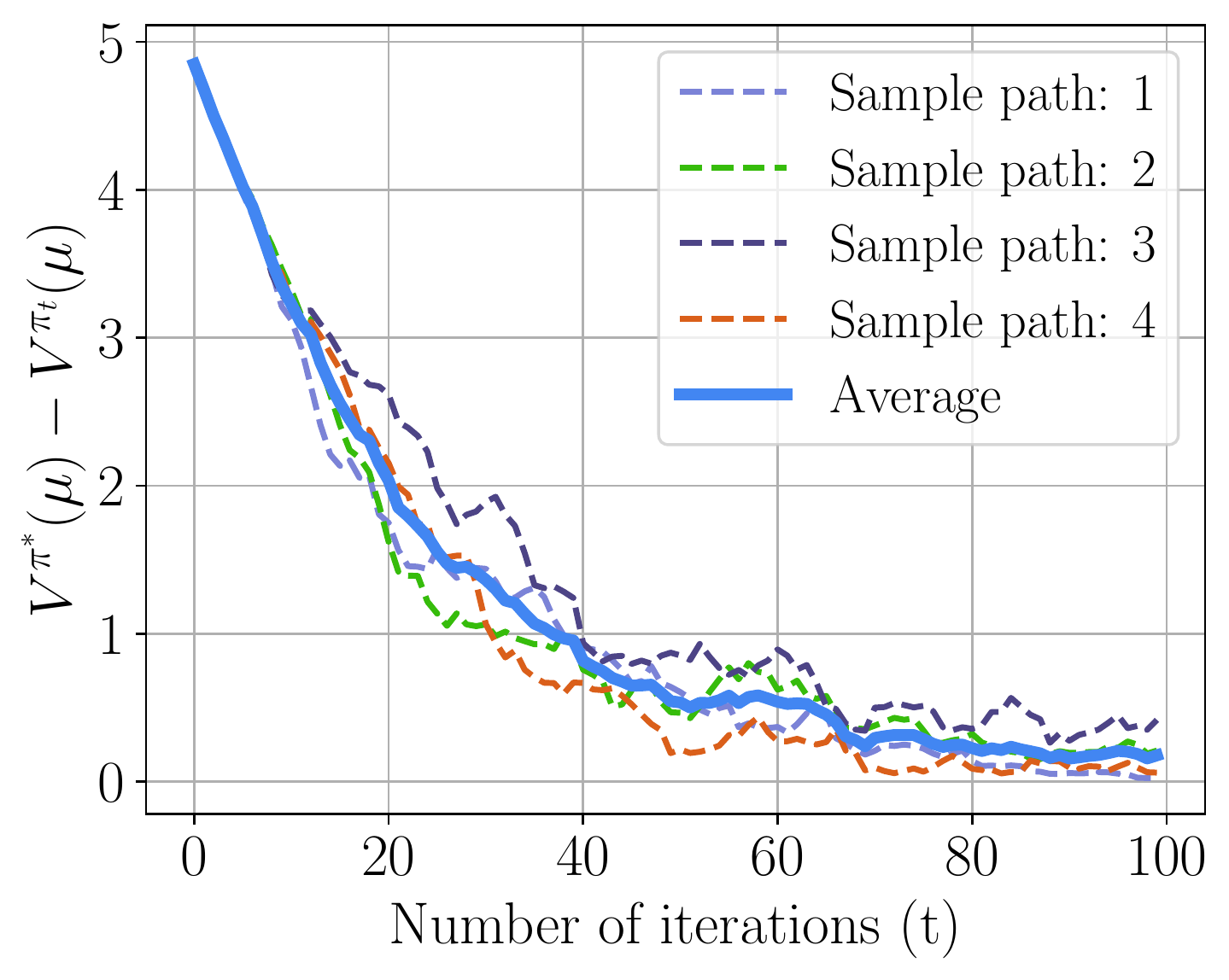}
    \caption{Convergence of Algorithm \ref{alg:off_policy_NPG} on a 5 state, 3 action MDP. Each dashed line is for one sample path of the algorithm, and the solid line is the average of the 4 sample paths. See Appendix \ref{sec:exp_result_details} for more details.}
    \label{fig:1}
\end{figure}

\subsection{Sample Complexity Analysis}\label{subsec:complexity}
With Theorem \ref{thm:main} at hand, we now analyze the sample complexity of off-policy NAC algorithm \ref{alg:off_policy_NPG}. 

\textit{Sample Complexity for Global Optimum.} Suppose that $\bar{\rho} \geq 1/\pi_{b,\min}$, where $\pi_{b,\min}:=\min_{s,a}\pi_b(a|s)$. In this case, we have $E_3=0$, i.e., the bias due to truncation is eliminated, and hence we have convergence to a global optimum. Theorem \ref{thm:main} implies the following sample complexity result, whose proof is presented in Appendix \ref{pf:co:sample_complexity}.

\begin{corollary}\label{co:sample_complexity}
In order to obtain an $\epsilon$-optimal policy, the total number of samples required (i.e., $TK$) is of the size
\begin{align*}
    \mathcal{O}(\epsilon^{-3}\log^2(1/\epsilon)) \Tilde{\mathcal{O}}((1-\gamma)^{-11}M_{\min}^{-3}\pi_{b,\min}^{-2}),
\end{align*}
where $M_{\min}=\min_{s,a}\mu_b(s)\pi_b(a|s)$.
\end{corollary}

The $\mathcal{O}(\epsilon^{-3}\log^2(1/\epsilon))$ dependence on the accuracy $\epsilon$ advances the state of the art results in on-policy NAC. See Table \ref{table: results} for more details. The dependence on the state-action space is at least $|\mathcal{S}|^3|\mathcal{A}|^5$, which is achieved when $\pi_b(a|s)=\frac{1}{|\mathcal{A}|}$ for all $a$ and $\mu_b(s)=\frac{1}{|\mathcal{S}|}$ for all $s$ (i.e., uniform exploration). The $\mathcal{\Tilde{O}}((1-\gamma)^{-11})$ dependence on the discount factor while seemingly loose, agrees with known results about NPG in \cite{agarwal2019theory} (Corollary 6.3). See Appendix \ref{ap:compute-sample-complexity} for more details about the comparison to \cite{agarwal2019theory}.

Note that in off-policy TD-learning, one set of samples can be used multiple times to evaluate different policies. Therefore, it is natural to consider repeatedly using the same set of samples in the critic (the $Q$-trace algorithm) in the off-policy NAC algorithm. In that case, the sample complexity is reduced from $KT=\Tilde{\mathcal{O}}(\epsilon^{-3})$ to only $K=\Tilde{\mathcal{O}}(\epsilon^{-2})$. Although this approach seems reasonable, numerical experiments suggest that it may lead to the divergence of Algorithm \ref{alg:off_policy_NPG}. See Appendix \ref{sec:exp_result_details} for more details.

\section{Proof Sketch of Our Main Results}\label{sec:proof}

In this section, we present the key steps in proving Theorems \ref{thm:Q-trace} and \ref{thm:main}.

\subsection{Proof Sketch of Theorem \ref{thm:Q-trace}}
To prove Theorem \ref{thm:Q-trace}, we begin by introducing some notations. For any $k\geq 0$, let $X_k=(S_k,A_k,...,S_{k+n})$. It is clear that $\{X_k\}$ is a Markov chain, whose state-space is denoted by $\mathcal{X}$. Moreover, under Assumption \ref{as:MC}, the Markov chain $\{X_k\}$ has a unique stationary distribution, denoted by $\mu_X$. Let $\mathcal{T}:\mathbb{R}^{|\mathcal{S}||\mathcal{A}|}\times\Pi\times\mathcal{X}\mapsto\mathbb{R}^{|\mathcal{S}||\mathcal{A}|}$ be an operator defined by 
\begin{align*}
    &[\mathcal{T}(Q,\pi,x)](s,a)=[\mathcal{T}(Q,\pi,s_0,a_0,...,s_n)](s,a)\\
    =\;&\mathbb{I}_{(s,a)=(s_0,a_0)}\sum_{i=0}^{n-1}\gamma^i\prod_{j=1}^ic_\pi(s_j,a_j)(\mathcal{R}(s_i,a_i)+\gamma \rho_\pi(s_{i+1},a_{i+1})Q(s_{i+1},a_{i+1})-Q(s_i,a_i))+Q(s,a)
\end{align*}
for all $(s,a)$. We further define $\mathcal{T}_e:\mathbb{R}^{|\mathcal{S}||\mathcal{A}|}\times\Pi\mapsto\mathbb{R}^{|\mathcal{S}||\mathcal{A}|}$ by $\mathcal{T}_e(Q,\pi)=\mathbb{E}_{X\sim \mu_X}\mathcal{T}(Q,\pi,X)$, which can be viewed as the expected version of the operator $\mathcal{T}$.

Using the notation given above, the $Q$-trace update equation (Algorithm \ref{alg:Q-trace} line 5) can be equivalently written by
\begin{align}
    Q_{k+1}=\;&Q_k+\alpha (\mathcal{T}(Q_k,\pi,X_k)-Q_k)\label{alg:Q-trace-remodel}\\
    =\;&Q_k+\alpha (\mathcal{T}_e(Q_k,\pi)-Q_k)+\alpha(\mathcal{T}(Q_k,\pi,X_k)-\mathcal{T}_e(Q_k,\pi))\tag{$*$}
\end{align}
The above update equation can be viewed as a stochastic approximation algorithm for solving the fixed-point equation $\mathcal{T}_e(Q,\pi)=Q$ with Markovian noise. To see this, assume for the moment that the term ($*$) is identically zero. Then the Algorithm is the fixed-point iteration for solving the equation $\mathcal{T}_e(Q,\pi)=Q$, and it is known to converge when the operator $\mathcal{T}_e(\cdot,\pi)$ is a \textit{contraction mapping} \cite{banach1922operations}. Now in the presence of the term $(*)$, the algorithm becomes a Markovian stochastic approximation algorithm for solving $\mathcal{T}_e(Q,\pi)=Q$.

Intuitively, once we show the desired contraction property of the operator $\mathcal{T}_e(\cdot,\pi)$ and have control on the error caused by the Markovian noise $(*)$, we should be able to establish the convergence bounds of Algorithm (\ref{alg:Q-trace-remodel}). In order to show such properties, we need the following notation.
\begin{enumerate}[(1)]
    \item Let $\pi_{\bar{c}}$ and $\pi_{\bar{\rho}}$ be two policies defined by
    \begin{align*}
        \pi_{\bar{c}}(a|s)=\frac{\min(\bar{c}\pi_b(a|s),\pi(a|s))}{\sum_{a'}\min(\bar{c}\pi_b(a'|s),\pi(a'|s))}\quad\text{and}\quad \pi_{\bar{\rho}}(a|s)=\frac{\min(\bar{\rho}\pi_b(a|s),\pi(a|s))}{\sum_{a'}\min(\bar{\rho}\pi_b(a'|s),\pi(a'|s))},\;\forall\;(s,a).
    \end{align*}
    \item Let $C_\pi,D_\pi\in\mathbb{R}^{|\mathcal{S}||\mathcal{A}|\times |\mathcal{S}||\mathcal{A}|}$ be diagonal matrices s.t. $C_\pi((s,a),(s,a))=\sum_{a}\min(\bar{c}\pi_b(a|s),\pi(a|s))$ and $D_\pi((s,a),(s,a))=\sum_{a}\min(\bar{\rho}\pi_b(a|s),\pi(a|s))$
    for all $(s,a)$. Let $C_{\min}=\bar{c}\min_{s,a}\pi_b(a|s)$. Note that we have $C_{\min} I\leq C_\pi\leq D_\pi\leq I$ (component-wise).
    \item Let $P_\pi\in\mathbb{R}^{|\mathcal{S}||\mathcal{A}|\times |\mathcal{S}||\mathcal{A}|}$ be a stochastic matrix defined by $P_\pi((s,a),(s',a'))=P_a(s,s')\pi(a'|s')$, i.e., the probability of transition from $(s,a)$ to $(s',a')$ under policy $\pi$. Let $R$ be a vector in $\mathbb{R}^{|\mathcal{S}||\mathcal{A}|}$ such that $R(s,a)=\mathcal{R}(s,a)$ for all $(s,a)$.
    \item Let $M\in\mathbb{R}^{|\mathcal{S}||\mathcal{A}|\times |\mathcal{S}||\mathcal{A}|}$ be a diagonal matrix such that $M((s,a),(s,a))=\mu_b(s)\pi_b(a|s)$, which is the steady-state probability of visiting $(s,a)$. Let $M_{\min}=\min_{s,a}\mu_b(s)\pi_b(a|s)$. Note that  $0<M_{\min}<1$ under Assumption \ref{as:MC}.
\end{enumerate}

Now we are ready to establish the desired properties of Algorithm (\ref{alg:Q-trace-remodel}) in the following proposition, whose proof is presented in Appendix \ref{pf:prop:Q-trace-property}.

\begin{proposition}\label{prop:Q-trace-property}
The following properties hold regarding the operators $\mathcal{T}(\cdot)$, $\mathcal{T}_e(\cdot)$, and the Markov chain $\{X_k\}$.
\begin{enumerate}[(1)]
    \item The operator $\mathcal{T}(\cdot)$ satisfies $\|\mathcal{T}(Q_1,\pi,x)-\mathcal{T}(Q_2,\pi,x)\|_\infty\leq 2(\bar{\rho}+1)f(\bar{c},\gamma)\|Q_1-Q_2\|_\infty$ and $\|\mathcal{T}(\bm{0},\pi,x)\|_\infty\leq f(\bar{c},\gamma)$ for any $Q_1,Q_2\in\mathbb{R}^{|\mathcal{S}||\mathcal{A}|}$, $\pi\in \Pi$, and $x\in\mathcal{X}$.
    \item For all $k\geq 0$, it holds that 
    \begin{align*}
        \max_{x\in\mathcal{X}}\|P^{k+n+1}(x,\cdot)-\mu_X(\cdot)\|_{\text{TV}}\leq Cu^k,
    \end{align*}
    where $\|\cdot\|_{\text{TV}}$ is the \textit{total variation distance}.
    \item The operator $\mathcal{T}_e(\cdot)$ has the following properties:
    \begin{enumerate}[(a)]
        \item $\mathcal{T}_e(\cdot,\pi)$ is 
a linear operator given by $\mathcal{T}_e(Q,\pi)=AQ+b$, where $A=I-\sum_{i=0}^{n-1}\gamma^iM(P_{\pi_{\bar{c}}}C_\pi)^{i}(I-\gamma P_{\pi_{\bar{\rho}}}D_\pi)$ and $b=\sum_{i=0}^{n-1}\gamma^iM(P_{\pi_{\bar{c}}}C_\pi)^{i}R$.
\item $\mathcal{T}_e(\cdot,\pi)$ is 
a \textit{contraction mapping} with respect to $\|\cdot\|_\infty$, with contraction factor
\begin{align*}
    \gamma_c=1-\frac{M_{\min}(1-\gamma)(1-(\gamma C_{\min})^n)}{1-\gamma C_{\min}}.
\end{align*}
\item $\mathcal{T}_e(\cdot,\pi)$ has a unique fixed-point $Q^{\bar{\rho},\pi}$, which is the unique solution to the modified Bellman's equation $Q=R+\gamma P_{\pi_{\bar{\rho}}} D_\pi Q$.
    \end{enumerate}
\end{enumerate}
\end{proposition}

Several remarks are in order. First, using Proposition \ref{prop:Q-trace-property} (1), we have by triangle inequality that
\begin{align*}
    \|\mathcal{T}(Q,\pi,x)\|_\infty\leq 2f(\bar{c},\gamma)((\bar{\rho}+1)\|Q\|_\infty+1)
\end{align*}
for any $Q$, $\pi$ and $x$. This is important to control the Markovian noise as it implies that the noisy operator $\|\mathcal{T}(Q_k,\pi,X_k)\|_\infty$ is at most an affine function of $\|Q_k\|_\infty$.

Proposition \ref{prop:Q-trace-property} (2) implies that the Markov chain $\{X_k\}$ mixes geometrically fast, which is also an important property we need to control the Markovian noise.

Proposition \ref{prop:Q-trace-property} (3) establishes all the desired properties for the expected operator $\mathcal{T}_e(\cdot)$. First of all, $\mathcal{T}_e(\cdot,\pi)$ is a contraction operator, with a contraction factor $\gamma_c$ independent of the target policy $\pi$. This uniform contraction property is necessary for us to combine the critic with the actor later in Section \ref{subsec:combine}, as the policy $\pi_t$ is time-varying. 

Note that from Proposition \ref{prop:Q-trace-property} (3) (c) we see that when $\bar{\rho}\geq \max_{s,a}\frac{\pi(a|s)}{\pi_b(a|s)}$, such modified Bellman's equation becomes the regular Bellman's equation for $Q^\pi$, and hence we have $Q^{\bar{\rho},\pi}=Q^\pi$, which agrees with Lemma \ref{le:bias}. 

The above proposition enables us to interpret Eq. (\ref{alg:Q-trace-remodel}) as a Markovian Stochastic Approximation involving a contraction mapping. Theorem \ref{thm:Q-trace} then follows from using finite-sample bounds on Markovian Stochastic Approximation established in \cite{chen2021finite}. See Appendix \ref{pf:thm:Q-trace} for the detailed proof.

\subsection{Proof Sketch of Theorem \ref{thm:main}}

The high level idea of proving Theorem \ref{thm:main} is as follows. We first analyze the iterates $\{\pi_t\}$ updated by the actor in Algorithm \ref{alg:off_policy_NPG}. The performance bound of $\pi_t$ would involve the error in the critic estimate, i.e., the difference between $Q_{t+1}$ and $Q^{\pi_t}$. We then use Corollary \ref{co:critic} of the $Q$-trace algorithm \ref{alg:Q-trace} to control the critic estimation error and finish the proof of Theorem \ref{thm:main}.

\subsubsection{Analysis of the Actor}

By analyzing the update of the actor, we obtain the performance bound of $\{\pi_t\}$ in the following proposition.

\begin{proposition}\label{prop:actor}
Consider iterates $\{\pi_t\}$ of Algorithm \ref{alg:off_policy_NPG}. We have for any $T\geq 1$:
\begin{align*}
    V^{\pi^*}(\mu)-\max_{0\leq t\leq T-1}\mathbb{E}\left[V^{\pi_t}(\mu)\right] \leq \underbrace{\frac{\log(e|\mathcal{A}|)}{(1-\gamma)^2\beta T}}_{\text{Error in the actor}}+\underbrace{\frac{4}{(1-\gamma)^2T}\sum_{t=0}^{T-1}\mathbb{E}[\|Q^{\pi_t}-Q_{t+1}\|_\infty]}_{\text{Error in the Critic}}.
\end{align*}
\end{proposition}

The proof of Proposition \ref{prop:actor} is inspired by that of Theorem 5.3 in \cite{agarwal2019theory}, and is presented in Appendix \ref{pf:prop:actor}. The main difference is that in \cite{agarwal2019theory} they assume access to the dynamics of the underlying MDP. Hence they can directly use the $Q$-function $Q^{\pi_t}$ in the policy update. Here in the RL setting, we can only use the noisy estimate $Q_t$ to perform the policy update. As a consequence, when compared to Theorem 5.3 of \cite{agarwal2019theory}, we have the critic error term $\frac{4}{(1-\gamma)^2T}\sum_{t=0}^{T-1}\mathbb{E}[\|Q^{\pi_t}-Q_{t+1}\|_\infty]$ on the RHS of the resulting inequality of Proposition \ref{prop:actor}.

\subsubsection{Combining the Actor and the Critic}\label{subsec:combine}
In view of Proposition \ref{prop:actor}, what remains to do in proving Theorem \ref{thm:main} is to apply Corollary \ref{co:critic} to control the error term $\mathbb{E}[\|Q^{\pi_t}-Q_{t+1}\|_\infty]$ for any $0\leq t\leq T-1$. However, there is a challenge in doing this. Corollary \ref{co:critic} and Theorem \ref{thm:Q-trace} are stated for a fixed target policy $\pi$, while in Algorithm \ref{alg:off_policy_NPG} the policies $\pi_t$ are stochastic. We overcome 
this challenge by using a conditioning argument and exploiting Markovian nature of the samples. The full details are presented in Appendix \ref{pf:thm:main}.

\section{Conclusion and Future Work} \label{sec:con}

In this work, we study the convergence bounds of NAC, where the critic uses the $Q$-trace algorithm to perform off-policy learning. Such off-policy NAC algorithm enables us to overcome the difficulty of exploration in on-policy NAC, and establish the convergence bounds under minimal assumptions. A future direction is to extend our results to the case where function approximation is used. Note that off-policy TD with function approximation can be unstable in general \cite{sutton2018reinforcement}. The first step in this direction is to modify the algorithm to achieve convergence.

\bibliographystyle{imsart-number}
\bibliography{references}    

\newpage

\begin{appendix}
\section{The Q-Trace Algorithm}

\subsection{Proof of Proposition \ref{prop:Q-trace-property}}\label{pf:prop:Q-trace-property}
\begin{enumerate}[(1)]
    \item Using the definition of the operator $\mathcal{T}(\cdot)$, we have for any $Q_1,Q_2$, $\pi$, $x=(s_0,a_0,...,s_n,a_n)\in\mathcal{X}$, and state-action pairs $(s,a)$:
\begin{align*}
    &|[\mathcal{T}(Q_1,\pi,x)](s,a)-[\mathcal{T}(Q_2,\pi,x)](s,a)|\\
	=\;&\bigg|\mathbb{I}_{\{(s,a)=(s_0,a_0)\}}\sum_{i=0}^{n-1}\gamma^i\left(\prod_{j=1}^{i}c_\pi(s_j,a_j)\right)\left(\gamma \rho_\pi(s_{i+1},a_{i+1}) [Q_1-Q_2](s_{i+1},a_{i+1})-[Q_1-Q_2](s_i,a_i)\right)\\
	&+[Q_1-Q_2](s,a)\bigg|\\
	\leq \;&\sum_{i=0}^{n-1}\gamma^i\left(\prod_{j=1}^{i}c_\pi(s_j,a_j)\right)\left(\gamma \rho_\pi(s_{i+1},a_{i+1})+1\right)\|Q_1-Q_2\|_\infty+\|Q_1-Q_2\|_\infty\\
	\leq \;&\sum_{i=0}^{n-1}(\gamma \bar{c})^{i}(\bar{\rho}+1)\|Q_1-Q_2\|_\infty+\|Q_1-Q_2\|_\infty\tag{$c_\pi(s,a)\leq \bar{c}$ and $\rho_\pi(s,a)\leq \bar{\rho}$ for any $(s,a)$}\\
	\leq \;&\begin{dcases}
		2n(\bar{\rho}+1)\|Q_1-Q_2\|_\infty,&\gamma\bar{c}=1,\\
		\frac{2(\bar{\rho}+1)(1-(\gamma\bar{c})^n)}{1-\gamma\bar{c}}\|Q_1-Q_2\|_\infty,&\gamma\bar{c}\neq 1.
	\end{dcases}
\end{align*}
It follows that $\|\mathcal{T}(Q_1,\pi,x)-\mathcal{T}(Q_2,\pi,x)\|_\infty\leq 2f(\bar{c},\gamma)(\bar{\rho}+1)\|Q_1-Q_2\|_\infty$. Similarly, for any $\pi\in \Pi$ and $x=(s_0,a_0,...,s_n,a_n)\in\mathcal{X}$, we have for any $(s,a)$:
\begin{align*}
	|[\mathcal{T}(\bm{0},\pi,x)](s,a)|&=\left|\mathbb{I}_{\{(s,a)=(s_0,a_0)\}}\sum_{i=0}^{n-1}\gamma^i\left(\prod_{j=1}^{i}c_\pi(s_j,a_j)\right)\mathcal{R}(s_i,a_i)\right|\\
	&\leq \sum_{i=0}^{n-1}\gamma^i\left(\prod_{j=1}^{i}c_\pi(s_j,a_j)\right)\tag{$\mathcal{R}(s,a)\in [0,1]$ for any $(s,a)$}\\
	&\leq \sum_{i=0}^{n-1}(\gamma\bar{c})^{i}\tag{$c_\pi(s,a)\leq \bar{c}$ for any $(s,a)$}\\
	&=\begin{dcases}
		n,&\gamma\bar{c}=1,\\
		\frac{1-(\gamma\bar{c})^n}{1-\gamma\bar{c}},&\gamma\bar{c}\neq 1.
	\end{dcases}
\end{align*}
Hence we have $\|\mathcal{T}(\bm{0},\pi,x)\|_\infty\leq f(\bar{c},\gamma)$.
\item Since the Markov chain $\{S_k\}$ induced by the behavior policy $\pi_b$ is irreducible and aperiodic, there exists $C>0$ and $u\in (0,1)$ such that $\max_{s\in\mathcal{S}}\|P^k(s,\cdot)-\mu_b(\cdot)\|_{\text{TV}}\leq Cu^k$ for all $k\geq 0$ \cite{levin2017markov}, where $P^k$ represents the $k$-step transition probability matrix. Now consider the Markov chain $\{X_k\}$. We have for all $k\geq 0$:
\begin{align*}
    &\max_{x\in\mathcal{X}}\left\|P^{k+n+1}(x,\cdot)-\mu_X(\cdot)\right\|_{\text{TV}}\\
    = \;&\frac{1}{2}\max_{s_0,a_0,...,s_n,a_n}\sum_{s_0',a_0',...,s_n',a_n'}\left|\sum_{s}P_{a_n}(s_n,s)P^k(s,s_0')-\mu_b(s_0')\right|\pi_b(a_0'|s_0')\prod_{i=0}^{n-1} P_{a_i'}(s_i',s_{i+1}')\pi_b(a_{i+1}'|s_{i+1}')\tag{$P_a$ is the transition probability matrix under action $a$}\\
    \leq \;&\frac{1}{2}\max_{s_n,a_n}\sum_{s_0'}\left|\sum_{s}P_{a_n}(s_n,s)P^k(s,s_0')-\mu_b(s_0')\right|\\
    = \;&\frac{1}{2}\max_{s_n,a_n}\sum_{s}P_{a_n}(s_n,s)\sum_{s_0'}\left|P^k(s,s_0')-\mu_b(s_0')\right|\\
    \leq \;&\frac{1}{2}\max_{s}\sum_{s_0'}\left|P^k(s,s_0')-\mu_b(s_0')\right|\\
    =\;&\max_{s\in\mathcal{S}}\left\|P^k(s,\cdot)-\mu_b(\cdot)\right\|_{\text{TV}}\\
    \leq \;&Cu^k.
\end{align*}
\item 
\begin{enumerate}[(a)]
    \item We first compute $\mathcal{T}_e(Q,\pi)$ in the following. For any $Q\in\mathbb{R}^{|\mathcal{S}||\mathcal{A}|}$ and $\pi\in \Pi$, we have for any $(s,a)$:
	\begin{align*}
		&[\mathcal{T}_e(Q,\pi)](s,a)\\
		=\;&\mathbb{E}_{S_0\sim \mu_b}\Bigg[\mathbb{I}_{\{(s,a)=(S_0,A_0)\}}\sum_{i=0}^{n-1}\gamma^i\left(\prod_{j=1}^{i}c_\pi(S_j,A_j)\right)(\mathcal{R}(S_i,A_i)+\gamma \rho_\pi(S_{i+1},A_{i+1}) Q(S_{i+1},A_{i+1})\\
		&-Q(S_i,A_i))\Bigg]+Q(s,a).
	\end{align*}
	For any $0\leq i\leq n-1$, we have
	\begin{align*}
		&\mathbb{E}_{S_0\sim \mu_b}\left[\mathbb{I}_{\{(s,a)=(S_0,A_0)\}}\gamma^i\left(\prod_{j=1}^{i}c_\pi(S_j,A_j)\right)\left(\mathcal{R}(S_i,A_i)+\gamma \rho_\pi(S_{i+1},A_{i+1}) Q(S_{i+1},A_{i+1})-Q(S_i,A_i)\right)\right]\\
		=\;&\mathbb{E}_{S_0\sim \mu_b}\Bigg[\mathbb{I}_{\{(s,a)=(S_0,A_0)\}}\gamma^i\left(\prod_{j=1}^{i}c_\pi(S_j,A_j)\right)(\mathcal{R}(S_i,A_i)\\
		&+\gamma \mathbb{E}[\rho_\pi(S_{i+1},A_{i+1}) Q(S_{i+1},A_{i+1})\mid S_0\sim \mu_b,A_0,...,S_i,A_i]-Q(S_i,A_i))\Bigg]\\
		=\;&\mathbb{E}_{S_0\sim \mu_b}\Bigg[\mathbb{I}_{\{(s,a)=(S_0,A_0)\}}\gamma^i\left(\prod_{j=1}^{i}c_\pi(S_j,A_j)\right)(\mathcal{R}(S_i,A_i)\\
		&+\gamma \sum_{s',a'}P_{A_i}(S_i,s')\pi_b(a'|s')\min\left(\bar{\rho},\frac{\pi(a'|s')}{\pi_b(a'|s')}\right)Q(s',a')-Q(S_i,A_i))\Bigg]\\
		=\;&\mathbb{E}_{S_0\sim \mu_b}\Bigg[\mathbb{I}_{\{(s,a)=(S_0,A_0)\}}\gamma^i\left(\prod_{j=1}^{i}c_\pi(S_j,A_j)\right)(\mathcal{R}(S_i,A_i)\\
		&+\gamma \sum_{s',a'}P_{A_i}(S_i,s')\min\left(\bar{\rho}\pi_b(a'|s'),\pi(a'|s')\right)Q(s',a')-Q(S_i,A_i))\Bigg]\\
		=\;&\mathbb{E}_{S_0\sim \mu_b}\Bigg[\mathbb{I}_{\{(s,a)=(S_0,A_0)\}}\gamma^i\left(\prod_{j=1}^{i}c_\pi(S_j,A_j)\right)(\mathcal{R}(S_i,A_i)\\
		&+\gamma \sum_{s',a'}P_{A_i}(S_i,s')D_\pi(s',a')\frac{\min\left(\bar{\rho}\pi_b(a'|s'),\pi(a'|s')\right)}{\sum_{a'}\min\left(\bar{\rho}\pi_b(a'|s'),\pi(a'|s')\right)}Q(s',a')-Q(S_i,A_i))\Bigg]\\
		=\;&\mathbb{E}_{S_0\sim \mu_b}\Bigg[\mathbb{I}_{\{(s,a)=(S_0,A_0)\}}\gamma^i\left(\prod_{j=1}^{i}c_\pi(S_j,A_j)\right)(\mathcal{R}(S_i,A_i)\\
		&+\gamma \sum_{s',a'}P_{A_i}(S_i,s')D_\pi(s',a')\pi_{\bar{\rho}}(a'|s')Q(s',a')-Q(S_i,A_i))\Bigg]\\
		=\;&\mathbb{E}_{S_0\sim \mu_b}\Bigg[\mathbb{I}_{\{(s,a)=(S_0,A_0)\}}\gamma^i\left(\prod_{j=1}^{i}c_\pi(S_j,A_j)\right)(\mathcal{R}(S_i,A_i)\\
		&+\gamma \sum_{s',a'}P_{\pi_{\bar{\rho}}}((S_i,A_i),(s',a'))D_\pi(s',a')Q(s',a')-Q(S_i,A_i))\Bigg]\\
		=\;&\mathbb{E}_{S_0\sim \mu_b}\left[\mathbb{I}_{\{(s,a)=(S_0,A_0)\}}\gamma^i\left(\prod_{j=1}^{i}c_\pi(S_j,A_j)\right)(\mathcal{R}(S_i,A_i)+\gamma [P_{\pi_{\bar{\rho}}}D_\pi Q](S_i,A_i)-Q(S_i,A_i))\right]\\
		=\;&\mathbb{E}_{S_0\sim \mu_b}\Bigg[\mathbb{I}_{\{(s,a)=(S_0,A_0)\}}\gamma^i\left(\prod_{j=1}^{i-1}c_\pi(S_j,A_j)\right)\times\\
		&\mathbb{E}\left[c_\pi(S_i,A_i)(\mathcal{R}(S_i,A_i)+\gamma [P_{\pi_{\bar{\rho}}}D_\pi Q](S_i,A_i)-Q(S_i,A_i))\mid S_0\sim \mu_b,A_0,...,S_{i-1},A_{i-1}\right]\Bigg] \\
		=\;&\mathbb{E}_{S_0\sim \mu_b}\Bigg[\mathbb{I}_{\{(s,a)=(S_0,A_0)\}}\gamma^i\left(\prod_{j=1}^{i-1}c_\pi(S_j,A_j)\right)\times\\
		&\left(\sum_{s',a'}P_{A_{i-1}}(S_{i-1},s')\pi_b(a'|s')\min\left(\bar{c},\frac{\pi(a'|s')}{\pi_b(a'|s')}\right)[R+\gamma P_{\pi_{\bar{\rho}}}D_\pi Q-Q](s',a')\right)\Bigg]\\
		=\;&\mathbb{E}_{S_0\sim \mu_b}\Bigg[\mathbb{I}_{\{(s,a)=(S_0,A_0)\}}\gamma^i\left(\prod_{j=1}^{i-1}c_\pi(S_j,A_j)\right)\times\\
		&\left(\sum_{s',a'}P_{A_{i-1}}(S_{i-1},s')\min\left(\bar{c}\pi_b(a'|s'),\pi(a'|s')\right)[R+\gamma P_{\pi_{\bar{\rho}}}D_\pi Q-Q](s',a')\right)\Bigg]\\
		=\;&\mathbb{E}_{S_0\sim \mu_b}\Bigg[\mathbb{I}_{\{(s,a)=(S_0,A_0)\}}\gamma^i\left(\prod_{j=1}^{i-1}c_\pi(S_j,A_j)\right)\times\\
		&\left(\sum_{s',a'}P_{A_{i-1}}(S_{i-1},s')C_\pi(s',a')\pi_{\bar{c}}(a'|s')[R+\gamma P_{\pi_{\bar{\rho}}}D_\pi Q-Q](s',a')\right)\Bigg] \\
		=\;&\mathbb{E}_{S_0\sim \mu_b}\left[\mathbb{I}_{\{(s,a)=(S_0,A_0)\}}\gamma^i\left(\prod_{j=1}^{i-1}c_\pi(S_j,A_j)\right)[P_{\pi_{\bar{c}}}C_\pi (R+\gamma P_{\pi_{\bar{\rho}}}D_\pi Q-Q)](S_{i-1},A_{i-1})\right]\\
		=\;&\cdots\\
		=\;&\mathbb{E}_{S_0\sim \mu_b}\left[\mathbb{I}_{\{(s,a)=(S_0,A_0)\}}\gamma^i[(P_{\pi_{\bar{c}}}C_\pi)^i (R+\gamma P_{\pi_{\bar{\rho}}}D_\pi Q-Q)](S_{0},A_{0})\right]\\
		=\;&\gamma^i\mu_b(s)\pi_b(a|s)\left[(P_{\pi_{\bar{c}}}C_\pi)^i (R+\gamma P_{\pi_{\bar{\rho}}}D_\pi Q-Q)\right](s,a)\\
		=\;&\gamma^i\left[M(P_{\pi_{\bar{c}}}C_\pi)^i (R+\gamma P_{\pi_{\bar{\rho}}}D_\pi Q-Q)\right](s,a).
	\end{align*}
	It follows that
	\begin{align}
		\mathcal{T}_e(Q,\pi)
		&=\sum_{i=0}^{n-1}M(\gamma P_{\pi_{\bar{c}}}C_\pi)^i (R+\gamma P_{\pi_{\bar{\rho}}}D_\pi Q-Q)+Q\label{eq:operator1}\\
		&=\underbrace{\left(I-\sum_{i=0}^{n-1}M(\gamma P_{\pi_{\bar{c}}}C_\pi)^{i}(I-\gamma P_{\pi_{\bar{\rho}}}D_\pi)\right)}_{A}Q+\underbrace{\sum_{i=0}^{n-1}M(\gamma P_{\pi_{\bar{c}}}C_\pi)^{i}R}_{b}.\label{eq:operator2}
    \end{align}
    \item
    We now show the desired contraction property. For any $Q_1,Q_2$ and $\pi$, we have
    \begin{align*}
        \|\mathcal{T}_e(Q_1,\pi)-\mathcal{T}_e(Q_2,\pi)\|_\infty\leq \|A\|_\infty\|Q_1-Q_2\|_\infty.
    \end{align*}
Consider the matrix $A$, we can rewrite it by
	\begin{align}
		A
		&=\sum_{i=1}^{n}\gamma^{i}M(P_{\pi_{\bar{c}}}C_\pi)^{i-1}(P_{\pi_{\bar{\rho}}}D_\pi)-\sum_{i=0}^{n-1}\gamma^{i}M(P_{\pi_{\bar{c}}}C_\pi)^{i}+I\nonumber\\
		&=\gamma^{n}M(P_{\pi_{\bar{c}}}C_\pi)^{n-1}(P_{\pi_{\bar{\rho}}}D_\pi) +\sum_{i=1}^{n-1}\gamma^{i}M(P_{\pi_{\bar{c}}}C_\pi)^{i-1}(P_{\pi_{\bar{\rho}}}D_\pi-P_{\pi_{\bar{c}}}C_\pi)+(I-M)\label{eq:operator3}.
	\end{align}
	Since
	\begin{align*}
		\left[P_{\pi_{\bar{\rho}}}D_\pi-P_{\pi_{\bar{c}}}C_\pi\right]((s,a),(s',a'))=P_a(s,s')\left(\min(\bar{\rho}\pi_b(a'|s'),\pi(a'|s'))-\min(\bar{c}\pi_b(a'|s'),\pi(a'|s'))\right)\geq 0
	\end{align*}
	for any $(s,a)$ and $(s',a')$,
	the matrix $A$ has non-negative entries. Therefore, we have 
	\begin{align*}
		\|A\|_\infty&=\|A\bm{1}\|_\infty\tag{$\bm{1}=(1,1,...,1)^\top$}\\
		&=\left\|\bm{1}-\sum_{i=0}^{n-1}M(\gamma P_{\pi_{\bar{c}}}C_\pi)^{i}(I-\gamma P_{\pi_{\bar{\rho}}}D_\pi)\bm{1}\right\|_{\infty}\\
		&\leq 1-M_{\min}\sum_{i=0}^{n-1}(\gamma C_{\min})^{i}(1-\gamma)\\
		&= 1-\frac{M_{\min}(1-\gamma)(1-(\gamma C_{\min})^{n})}{1-\gamma C_{\min}},
	\end{align*}
	where in the first inequality we used $C_{\min} \bm{1}\leq C_\pi \bm{1}\leq D_\pi \bm{1}\leq \bm{1}$ (component-wise). It follows that $\mathcal{T}_e(\cdot,\pi)$ is a contraction mapping with respect to $\|\cdot\|_\infty$, with contraction factor
	\begin{align*}
		\gamma_c=1-\frac{M_{\min}(1-\gamma)(1-(\gamma C_{\min})^{n})}{1-\gamma C_{\min}}.
	\end{align*}
	\item 
	The existence and uniqueness of the fixed-point of $\mathcal{T}_e(\cdot,\pi)$ follows from Banach fixed-point theorem \cite{banach1922operations}. To characterize the fixed-point, it is enough to show the modified Bellman's equation
	\begin{align*}
	    R+\gamma P_{\pi_{\bar{\rho}}} D_\pi Q-Q=0
	\end{align*}
	has a unique solution, i.e., the matrix $I-\gamma P_{\pi_{\bar{\rho}}} D_\pi$ is invertible. This is followed from
	\begin{align*}
	    \|P_{\pi_{\bar{\rho}}} D_\pi\|_\infty=\|P_{\pi_{\bar{\rho}}} D_\pi\bm{1}\|_\infty\leq \|P_{\pi_{\bar{\rho}}} \bm{1}\|_\infty=1.
	\end{align*}
\end{enumerate}
\end{enumerate}

\subsection{Proof of Lemma \ref{le:bias}}\label{pf:le:bias}
\begin{enumerate}[(1)]
    \item We begin with the Bellman's equation for $Q^\pi$: $Q^\pi=R+\gamma P_\pi Q^\pi$, and the modified Bellman's equation for $Q^{\bar{\rho},\pi}$: $Q^{\bar{\rho},\pi}=R+\gamma P_{\pi_{\bar{\rho}}} D_\pi Q^{\bar{\rho},\pi}$. Take the difference between these two equations and we obtain:
\begin{align*}
	Q^{\bar{\rho},\pi}-Q^\pi&=\gamma P_{\pi_{\bar{\rho}}} D_\pi Q^{\bar{\rho},\pi}-\gamma P_\pi Q^\pi\\
	&=\gamma P_{\pi_{\bar{\rho}}} D_\pi Q^{\bar{\rho},\pi}-\gamma P_{\pi_{\bar{\rho}}}D_\pi Q^\pi+\gamma P_{\pi_{\bar{\rho}}}D_\pi Q^\pi-\gamma P_\pi Q^\pi\\
	&=\gamma P_{\pi_{\bar{\rho}}} D_\pi(Q^{\bar{\rho},\pi}-Q^\pi)+\gamma(P_{\pi_{\bar{\rho}}}D_\pi-P_\pi)Q^\pi.
\end{align*}
Therefore, we have
\begin{align*}
	\left\|Q^{\bar{\rho},\pi}-Q^\pi\right\|_\infty&=\left\|(I-\gamma P_{\pi_{\bar{\rho}}} D_\pi)^{-1}\gamma(P_{\pi_{\bar{\rho}}}D_\pi-P_\pi)Q^\pi\right\|_\infty\\
	&\leq \gamma\left\|(I-\gamma P_{\pi_{\bar{\rho}}} D_\pi)^{-1}\right\|_\infty\left\|P_{\pi_{\bar{\rho}}}D_\pi-P_\pi\right\|_\infty\left\|Q^\pi\right\|_\infty\\
	&\leq \frac{1}{1-\gamma}\left\|(I-\gamma P_{\pi_{\bar{\rho}}} D_\pi)^{-1}\right\|_\infty\left\|P_{\pi_{\bar{\rho}}}D_\pi-P_\pi\right\|_\infty.
\end{align*}
Since 
\begin{align*}
	[P_{\pi_{\bar{\rho}}}D_\pi-P_\pi]((s,a),(s',a'))&=P_a(s,s')(\min(\bar{\rho}\pi_b(a'|s'),\pi(a'|s'))-\pi(a'|s')) \\
	&=P_a(s,s')\min(\bar{\rho}\pi_b(a'|s')-\pi(a'|s'),0)\\
	&=-P_a(s,s')\max(\pi(a'|s')-\bar{\rho}\pi_b(a'|s'),0)\\
	&\leq 0,
\end{align*}
we have 
\begin{align*}
	\left\|P_{\pi_{\bar{\rho}}}D_\pi-P_\pi\right\|_\infty=\left\|(P_\pi-P_{\pi_{\bar{\rho}}}D_\pi)\bm{1}\right\|_\infty\leq \max_{(s,a)}\max(\pi(a|s)-\bar{\rho}\pi_b(a|s),0).
\end{align*}
As for the term $\left\|(I-\gamma P_{\pi_{\bar{\rho}}} D_\pi)^{-1}\right\|_\infty$, note that for any invertible matrix $G$ we have
\begin{align*}
    \left\|G^{-1}\right\|_\infty&=\max_{x\neq 0}\frac{\left\|G^{-1}x\right\|_\infty}{\|x\|_\infty}\\
    &=\max_{y\neq 0}\frac{\|y\|_\infty}{\|Gy\|_\infty}\tag{Change of variable}\\
    &=\max_{y:\|y\|_\infty=1}\frac{1}{\|Gy\|_\infty}\\
    &=\frac{1}{\min_{y:\|y\|_\infty=1}\|Gy\|_\infty}.
\end{align*}
Therefore, we obtain
\begin{align*}
    \|(I-\gamma P_{\pi_{\bar{\rho}}} D_\pi)^{-1}\|_\infty&=\frac{1}{\min_{y:\|y\|_\infty=1}\|(I-\gamma P_{\pi_{\bar{\rho}}} D_\pi)y\|_\infty}\\
    &\leq \frac{1}{1-\gamma\max_{y:\|y\|_\infty=1}\| P_{\pi_{\bar{\rho}}} D_\pi y\|_\infty}\\
    &\leq \frac{1}{1-\gamma}.
\end{align*}

It follows that 
\begin{align*}
	\|Q^{\bar{\rho},\pi}-Q^\pi\|_\infty&\leq\frac{1}{1-\gamma}\|(I-\gamma P_{\pi_{\bar{\rho}}} D_\pi)^{-1}\|_\infty\|P_{\pi_{\bar{\rho}}}D_\pi-P_\pi\|_\infty\\
	&\leq  \frac{1}{(1-\gamma)^2}\max_{(s,a)}\max(\pi(a|s)-\bar{\rho}\pi_b(a|s),0).
\end{align*}
\item
Similarly, we have
\begin{align*}
    \|Q^{\bar{\rho},\pi}\|_\infty=\|(I-\gamma P_{\pi_{\bar{\rho}}}D_\pi)^{-1}R\|_\infty\leq \|(I-\gamma P_{\pi_{\bar{\rho}}}D_\pi)^{-1}\|_\infty\|R\|_\infty\leq \frac{1}{1-\gamma}.
\end{align*}
\end{enumerate}

\subsection{Proof of Theorem \ref{thm:Q-trace}}\label{pf:thm:Q-trace}
We begin by restating Theorem \ref{thm:Q-trace} in full details:
\begin{theorem}\label{thm:Q-trace-detail}
Consider $Q_k$ of Algorithm \ref{alg:Q-trace}. Suppose Assumption \ref{as:MC} is satisfied and the constant stepsize $\alpha$ is chosen such that $\alpha(\tau_\alpha+n+1)\leq \min\left(\frac{1}{12(\bar{\rho}+1)f(\bar{c},\gamma)},\frac{(1-\gamma_c)^2}{8208(\bar{\rho}+1)^2f(\bar{c},\gamma)^2\log(|\mathcal{S}||\mathcal{A}|)}\right)$. Then we have for all $k\geq \tau_\alpha+n+1$:
\begin{align*}
     \mathbb{E}[\|Q_k-Q^{\bar{\rho},\pi}\|_\infty^2]
     \leq \;&3\left(\|Q_0-Q^{\bar{\rho},\pi}\|_\infty+\|Q^{\bar{\rho},\pi}\|_\infty+1\right)^2\left(1-\frac{1-\gamma_c}{2}\alpha\right)^{k-(\tau_\alpha+n+1)}\\
     &+\frac{8208e\log (|\mathcal{S}||\mathcal{A}|)}{(1-\gamma_c)^2}(\bar{\rho}+1)^2f(\bar{c},\gamma)^2 (\|Q^{\bar{\rho},\pi}\|_\infty+1)^2\alpha (\tau_\alpha+n+1).
\end{align*}
\end{theorem}

\begin{proof}[Proof of Theorem \ref{thm:Q-trace-detail}]
To prove Theorem \ref{thm:Q-trace-detail}, we will apply the results in \cite{chen2021finite}. For self-containedness, we here restate Theorem 2.1 of \cite{chen2021finite} in the following. 

\begin{theorem}[Theorem 2.1 in \cite{chen2021finite}]\label{thm:chen2021}
Consider $\{x_k\}$ generated by the following stochastic approximation algorithm: $x_{k+1}=x_k+\epsilon (F(x_k,Y_k)-x_k)$. Suppose that
\begin{enumerate}[(1)]
    \item The random process $\{Y_k\}$ is a Markov chain with finite state-space $\mathcal{Y}$. Moreover, $\{Y_k\}$ has a unique stationary distribution $\mu_Y$ and there exist $C_1>0$ and $u_1\in (0,1)$ such that $\max_{y\in\mathcal{Y}}\|P^k(y,\cdot)-\mu_Y(\cdot)\|_{\text{TV}}\leq C_1 u_1^k$ for all $k\geq 0$.
    \item The operator $F:\mathbb{R}^d\times \mathcal{Y}\mapsto\mathbb{R}^d$ satisfies $\|F(x_1,y)-F(x_2,y)\|_\infty\leq A_1\|x_1-x_2\|_\infty$ and $\|F(\bm{0},y)\|_\infty\leq B_1$ for any $x_1,x_2\in\mathbb{R}^d$ and $y\in\mathcal{Y}$.
    \item The expected operator $\bar{F}:\mathbb{R}^d\mapsto\mathbb{R}^d$ defined by $\bar{F}(x)=\mathbb{E}_{Y\sim \mu_Y}[F(x,Y)]$ is a $\gamma_c'$ -- contraction mapping with respect to  $\|\cdot\|_\infty$. Denote the unique fixed-point of $\bar{F}(\cdot)$ by $x^*$. 
    \item The constant stepsize $\epsilon$ is chosen such that $\epsilon t_\epsilon\leq \min \left(\frac{1}{4(A_1+1)},\frac{(1-\gamma_c')^2}{912 (A_1+1)^2\log(d)}\right)$, where $t_\epsilon=\min\{k\geq 0\;:\;\max_{y\in\mathcal{Y}}\|P^k(y,\cdot)-\mu_Y(\cdot)\|_{\text{TV}}\leq \epsilon\}$.
\end{enumerate}
Then the following inequality holds for all $k\geq t_\epsilon$:
\begin{align*}
    \mathbb{E}[\|x_k-x^*\|_\infty^2]\leq& 3\left(\|x_0-x^*\|_\infty+\|x_0\|_\infty+\frac{B_1}{A_1+1}\right)^2\left(1-\frac{1-\gamma_c'}{2}\epsilon\right)^{k-t_\epsilon}\\
    &+\frac{912e\log (d)}{(1-\gamma_c')^2} ((A_1+1)\|x^*\|_\infty+B_1)^2\epsilon t_\epsilon.
\end{align*}
\end{theorem}

Proposition \ref{prop:Q-trace-property} enables us to apply Theorem \ref{thm:chen2021} to the $Q$-trace algorithm. Therefore, when the constant stepsize $\alpha$ is chosen such that $\alpha (\tau_\alpha+n+1)\leq \min\left(\frac{1}{12(\bar{\rho}+1)f(\bar{c},\gamma)},\frac{(1-\gamma_c)^2}{8208(\bar{\rho}+1)^2f(\bar{c},\gamma)^2\log(|\mathcal{S}||\mathcal{A}|)}\right)$ (which is always possible since $\alpha (\tau_\alpha+n+1)=\mathcal{O}(\alpha \log(1/\alpha))\rightarrow 0$ as $\alpha \rightarrow 0$), we have for all $k\geq \tau_\alpha+n+1$:
\begin{align*}
     \mathbb{E}[\|Q_k-Q^{\bar{\rho},\pi}\|_\infty^2]
     \leq \;&3\left(\|Q_0-Q^{\bar{\rho},\pi}\|_\infty+\|Q^{\bar{\rho},\pi}\|_\infty+1\right)^2\left(1-\frac{1-\gamma_c}{2}\alpha\right)^{k-(\tau_\alpha+n+1)}\\
     &+\frac{8208e\log (|\mathcal{S}||\mathcal{A}|)}{(1-\gamma_c)^2}(\bar{\rho}+1)^2f(\bar{c},\gamma)^2 (\|Q^{\bar{\rho},\pi}\|_\infty+1)^2\alpha (\tau_\alpha+n+1).
\end{align*}
\end{proof}

To go from Theorem \ref{thm:Q-trace-detail} to Theorem \ref{thm:Q-trace}, notice that $Q_0=\bm{0}$ in the $Q$-trace algorithm \ref{alg:Q-trace}, and $\|Q^{\bar{\rho},\pi}\|_\infty\leq \frac{1}{1-\gamma}$ for any $\bar{\rho}\geq 1$ and $\pi\in \Pi$ (Lemma \ref{le:bias}). Therefore, we have from Theorem \ref{thm:Q-trace-detail} that for all $k\geq \tau_\alpha+n+1$:
\begin{align*}
   \mathbb{E}[\|Q_k-Q^{\bar{\rho},\pi}\|_\infty^2]
     \leq \;&3\left(\|Q_0-Q^{\bar{\rho},\pi}\|_\infty+\|Q^{\bar{\rho},\pi}\|_\infty+1\right)^2\left(1-\frac{1-\gamma_c}{2}\alpha\right)^{k-(\tau_\alpha+n+1)}\\
     &+\frac{8208e\log (|\mathcal{S}||\mathcal{A}|)}{(1-\gamma_c)^2}(\bar{\rho}+1)^2f(\bar{c},\gamma)^2 (\|Q^{\bar{\rho},\pi}\|_\infty+1)^2\alpha (\tau_\alpha+n+1)\\
     =\;&3\left(2\|Q^{\bar{\rho},\pi}\|_\infty+1\right)^2\left(1-\frac{1-\gamma_c}{2}\alpha\right)^{k-(\tau_\alpha+n+1)}\\
     &+\frac{8208e\log (|\mathcal{S}||\mathcal{A}|)}{(1-\gamma_c)^2}(\bar{\rho}+1)^2f(\bar{c},\gamma)^2 (\|Q^{\bar{\rho},\pi}\|_\infty+1)^2\alpha (\tau_\alpha+n+1)\tag{$Q_0=\bm{0}$}\\
     \leq \;&\frac{27}{(1-\gamma)^2}\left(1-\frac{1-\gamma_c}{2}\alpha\right)^{k-(\tau_\alpha+n+1)}\\
     &+\frac{32832e\log (|\mathcal{S}||\mathcal{A}|)}{(1-\gamma_c)^2(1-\gamma)^2}(\bar{\rho}+1)^2f(\bar{c},\gamma)^2 \alpha (\tau_\alpha+n+1)\tag{$\|Q^{\bar{\rho},\pi}\|_\infty\leq \frac{1}{1-\gamma}$}\\
     =\;&\frac{c_1}{(1-\gamma)^2}\left(1-\frac{1-\gamma_c}{2}\alpha\right)^{k-(\tau_\alpha+n+1)}\\
     &+\frac{c_2\log (|\mathcal{S}||\mathcal{A}|)}{(1-\gamma_c)^2(1-\gamma)^2}(\bar{\rho}+1)^2f(\bar{c},\gamma)^2 \alpha (\tau_\alpha+n+1),
\end{align*}
where $c_1=27$ and $c_2=32832e$ are numerical constants.

\subsection{Proof of Corollary \ref{co:critic}}\label{pf:co:critic}
Under the same condition of Theorem  \ref{thm:Q-trace}, we have for all $k\geq \tau_\alpha+n+1$:
\begin{align*}
    \mathbb{E}[\|Q_k-Q^{\pi}\|_\infty]&\leq \mathbb{E}[\|Q_k-Q^{\bar{\rho},\pi}\|_\infty]+\|Q^{\bar{\rho},\pi}-Q^{\pi}\|_\infty\tag{Triangle inequality}\\
    &\leq \mathbb{E}[\|Q_k-Q^{\bar{\rho},\pi}\|_\infty]+\frac{\max(0,1-\bar{\rho}\min_{s,a}\pi_b(a|s))}{(1-\gamma)^2}\tag{Lemma \ref{le:bias}}\\
    &\leq  \left(\mathbb{E}[\|Q_k-Q^{\bar{\rho},\pi}\|_\infty^2]\right)^{1/2}+\frac{\max(0,1-\bar{\rho}\min_{s,a}\pi_b(a|s))}{(1-\gamma)^2}\tag{Jensen's inequality}\\
    &\leq (T_1+T_2)^{1/2}+\frac{\max(0,1-\bar{\rho}\min_{s,a}\pi_b(a|s))}{(1-\gamma)^2}\tag{Theorem \ref{thm:Q-trace}}\\
    &\leq \sqrt{T_1}+\sqrt{T_2}+\frac{\max(0,1-\bar{\rho}\min_{s,a}\pi_b(a|s))}{(1-\gamma)^2},\tag{$a^2+b^2\leq (a+b)^2$ for any $a,b\geq 0$}
\end{align*}
where $T_1$ and $T_2$ are given in Theorem \ref{thm:Q-trace}.

\section{Off-Policy Natural Actor-Critic Algorithm}

\subsection{Proof of Proposition \ref{prop:actor}}\label{pf:prop:actor}
To prove this proposition, it is more convenient to write the update equation for $\pi_{t}$ in Algorithm \ref{alg:off_policy_NPG} as
\begin{align*}
	\pi_{t+1}(a|s)=\pi_t(a|s)\frac{\exp(\beta( Q_{t+1}(s,a)-V^{\pi_t}(s)))}{\sum_{a'\in\mathcal{A}}\pi_t(a'|s)\exp(\beta(Q_{t+1}(s,a')-V^{\pi_t}(s)))}
\end{align*}
for all $(s,a)$. Denote $Z_{t}(s)=\sum_{a\in\mathcal{A}}\pi_t(a|s)\exp(\beta(Q_{t+1}(s,a)-V^{\pi_t}(s)))$. We first present a sequence of lemmas. The proofs are presented in Appendices \ref{pf:le:logzt}, \ref{pf:le:Vmu}, and \ref{pf:le:Vnu}. Throughout the paper, given an initial distribution $\mu$, we denote $d^t\equiv d_\mu^{\pi_t}$ and $d^*\equiv d_\mu^{\pi^*}$, where we omit $\mu$ for the ease of notation. 

\begin{lemma}\label{le:logzt}
The following inequality holds for all $t\geq 0$ and $s\in\mathcal{S}$:
\begin{align*}
	\log(Z_{t}(s))\geq\beta\sum_{a\in\mathcal{A}}\pi_t(a|s)(Q_{t+1}(s,a)-Q^{\pi_t}(s,a)).
\end{align*}
\end{lemma}
\begin{lemma}\label{le:Vmu}
Consider the iterates $\{\pi_t\}$ in Algorithm \ref{alg:off_policy_NPG}. The following inequality holds for any starting distribution $\mu$:
\begin{align*}
	V^{\pi_{t+1}}(\mu)-V^{\pi_t}(\mu)\geq\;&\frac{1}{1-\gamma}\mathbb{E}_{s\sim d^{t+1}}\sum_{a\in\mathcal{A}}(\pi_t(a|s)-\pi_{t+1}(a|s))(Q_{t+1}(s,a)-Q^{\pi_t}(s,a))\\
	&-\mathbb{E}_{s\sim \mu}\sum_{a\in\mathcal{A}}\pi_t(a|s)(Q_{t+1}(s,a)-Q^{\pi_t}(s,a))+\frac{1}{\beta}\mathbb{E}_{s\sim \mu}\log Z_{t}(s).
\end{align*}
\end{lemma}
\begin{lemma}\label{le:Vnu}
For any starting distribution $\mu$, we have for any $t\geq 0$:
\begin{align*}
	V^{\pi^*}(\mu)-V^{\pi_t}(\mu)=\;&\frac{1}{1-\gamma}\mathbb{E}_{s\sim d^*}\sum_{a\in\mathcal{A}}\pi^*(a|s)(Q^{\pi_t}(s,a)-Q_{t+1}(s,a))+\frac{1}{(1-\gamma)\beta}\mathbb{E}_{s\sim d^*}\log(Z_{t}(s))\\
	&+\frac{1}{(1-\gamma)\beta}\mathbb{E}_{s\sim d^*}\left[\mathcal{KL}(\pi^*(\cdot|s)\mid \pi_t(\cdot|s))-\mathcal{KL}(\pi^*(\cdot|s)\mid \pi_{t+1}(\cdot|s))\right].
\end{align*}
\end{lemma}

Now we proceed to prove Proposition \ref{prop:actor}. Letting $\mu=d^*$, we have by Lemma \ref{le:Vmu} that
\begin{align*}
    V^{\pi_{t+1}}(d^*)-V^{\pi_t}(d^*)\geq\;&\frac{1}{1-\gamma}\mathbb{E}_{s\sim d^{t+1}_{d^*}}\sum_{a\in\mathcal{A}}(\pi_t(a|s)-\pi_{t+1}(a|s))(Q_{t+1}(s,a)-Q^{\pi_t}(s,a))\\
	&-\mathbb{E}_{s\sim d^*}\sum_{a\in\mathcal{A}}\pi_t(a|s)(Q_{t+1}(s,a)-Q^{\pi_t}(s,a))+\frac{1}{\beta}\mathbb{E}_{s\sim d^*}\log Z_{t}(s).
\end{align*}
It follows that
\begin{align}\label{eq:10}
    \frac{1}{\beta}\mathbb{E}_{s\sim d^*}\log Z_{t}(s)\leq V^{\pi_{t+1}}(d^*)-V^{\pi_t}(d^*)+\frac{3}{1-\gamma}\|Q_{t+1}-Q^{\pi_t}\|_\infty.
\end{align}
Now for any $T\geq 1$, we have
\begin{align*}
	&\sum_{t=0}^{T-1}(V^{\pi^*}(\mu)-V^{\pi_t}(\mu))\\
	= \;&\frac{1}{1-\gamma}\sum_{t=0}^{T-1}\mathbb{E}_{s\sim d^*}\sum_{a\in\mathcal{A}}\pi^*(a|s)(Q^{\pi_t}(s,a)-Q_{t+1}(s,a))+\frac{1}{(1-\gamma)\beta}\sum_{t=0}^{T-1}\mathbb{E}_{s\sim d^*}\log(Z_{t}(s))\\
	&+\frac{1}{(1-\gamma)\beta}\sum_{t=0}^{T-1}\mathbb{E}_{s\sim d^*}\left[\mathcal{KL}(\pi^*(\cdot|s)\mid \pi_t(\cdot|s))-\mathcal{KL}(\pi^*(\cdot|s)\mid \pi_{t+1}(\cdot|s))\right]\\
	\leq \;&\frac{1}{1-\gamma}\sum_{t=0}^{T-1}\|Q^{\pi_t}-Q_{t+1}\|_\infty+\frac{1}{1-\gamma}\sum_{t=0}^{T-1}\left[V^{\pi_{t+1}}(d^*)-V^{\pi_t}(d^*) + \frac{3}{1-\gamma}\|Q^{\pi_t}-Q_{t+1}\|_\infty \right]\tag{Eq. (\ref{eq:10})}\\
	&+\frac{1}{(1-\gamma)\beta}\sum_{t=0}^{T-1}\mathbb{E}_{s\sim d^*}\left[\mathcal{KL}(\pi^*(\cdot|s)\mid \pi_t(\cdot|s))-\mathcal{KL}(\pi^*(\cdot|s)\mid \pi_{t+1}(\cdot|s))\right]\\
	\leq \;&\frac{1}{1-\gamma}\sum_{t=0}^{T-1}\|Q^{\pi_t}-Q_{t+1}\|_\infty+\frac{1}{1-\gamma}(V^{\pi_T}(d^*)-V^{\pi_0}(d^*))+\frac{3}{(1-\gamma)^2}\sum_{t=0}^{T-1}\|Q^{\pi_t}-Q_{t+1}\|_\infty\\
	&+\frac{1}{(1-\gamma)\beta}\mathbb{E}_{s\sim d^*}\left[\mathcal{KL}(\pi^*(\cdot|s)\mid \pi_0(\cdot|s))-\mathcal{KL}(\pi^*(\cdot|s)\mid \pi_{T}(\cdot|s))\right]\\
	\leq \;&\frac{4}{(1-\gamma)^2}\sum_{t=0}^{T-1}\|Q^{\pi_t}-Q_{t+1}\|_\infty+\frac{1}{(1-\gamma)^2}+\frac{\log(\mathcal{A})}{(1-\gamma)\beta}.
\end{align*}
Therefore, we have from the previous inequality:
\begin{align*}
	V^{\pi^*}(\mu)-\max_{0\leq t\leq T-1}\mathbb{E}\left[V^{\pi_t}(\mu)\right]&\leq V^{\pi^*}(\mu)-\frac{1}{T}\sum_{t=0}^{T-1}\mathbb{E}\left[V^{\pi_t}(\mu)\right]\\
	&\leq \frac{1}{(1-\gamma)^2T}+\frac{\log(|\mathcal{A}|)}{(1-\gamma)\beta T}+\frac{4}{(1-\gamma)^2T}\sum_{t=0}^{T-1}\mathbb{E}[\|Q^{\pi_t}-Q_{t+1}\|_\infty]\\
	&\leq \frac{\log(e|\mathcal{A}|)}{(1-\gamma)^2\beta T}+\frac{4}{(1-\gamma)^2T}\sum_{t=0}^{T-1}\mathbb{E}[\|Q^{\pi_t}-Q_{t+1}\|_\infty]
\end{align*}

\subsection{Proof of Theorem \ref{thm:main}}\label{pf:thm:main}

Our goal is to combine Proposition \ref{prop:actor} with  Corollary \ref{co:critic}. The only challenge remains is that Corollary \ref{co:critic} is stated for a fixed target policy $\pi$ while $\pi_t$ is stochastic. To overcome this difficulty, observe that $\pi_t$ is determined by $\{(S_k,A_k)\}_{0\leq k\leq t(K+n)}$ while $Q_{t+1}$ is determined by $\pi_t$ and $\{(S_k,A_k)\}_{t(K+n)\leq k\leq (t+1)(K+n)}$. Therefore, by the Markov property and the tower property of conditional expectation, we have for any $0\leq t\leq T-1$:
\begin{align}
    &\mathbb{E}[\|Q_{t+1}-Q^{\bar{\rho},\pi_t}\|_\infty]\nonumber\\
    =\;&\mathbb{E}\left[\mathbb{E}[\|Q_{t+1}-Q^{\bar{\rho},\pi_t}\|_\infty\mid S_0,A_0,...,S_{t(K+n)}.A_{t(K+n)}]\right]\nonumber\\
    \leq \;& \sqrt{T_1}+\sqrt{T_2}+\frac{\max(0,1-\bar{\rho}\min_{s,a}\pi_b(a|s))}{(1-\gamma)^2}\nonumber\\
    \leq \;& \frac{\sqrt{c_1}}{1-\gamma}\left(1-\frac{1-\gamma_c}{2}\alpha\right)^{\frac{1}{2}[K-(\tau_\alpha+n+1)]}+\frac{\sqrt{c_2}\log^{1/2} (|\mathcal{S}||\mathcal{A}|)}{(1-\gamma_c)(1-\gamma)}(\bar{\rho}+1)f(\bar{c},\gamma) [\alpha (\tau_\alpha+n+1)]^{1/2}\nonumber\\
    &+\frac{\max(0,1-\bar{\rho}\min_{s,a}\pi_b(a|s))}{(1-\gamma)^2}\nonumber\\
    \leq \;&\frac{6}{1-\gamma}\left(1-\frac{1-\gamma_c}{2}\alpha\right)^{\frac{1}{2}[K-(\tau_\alpha+n+1)]}+\frac{300\log^{1/2} (|\mathcal{S}||\mathcal{A}|)}{(1-\gamma_c)(1-\gamma)}(\bar{\rho}+1)f(\bar{c},\gamma) [\alpha (\tau_\alpha+n+1)]^{1/2}\nonumber\\
    &+\frac{\max(0,1-\bar{\rho}\min_{s,a}\pi_b(a|s))}{(1-\gamma)^2}\label{eq:71}
\end{align}
where in the last line we used $c_1=27$ and $c_2=32832e$.

Using Eq. (\ref{eq:71}) in Proposition \ref{prop:actor}, we have for all $T\geq 1$:
\begin{align*}
	&V^{\pi^*}(\mu)-\max_{0\leq t\leq T-1}\mathbb{E}\left[V^{\pi_t}(\mu)\right]\\
	\leq\;& V^{\pi^*}(\mu)-\frac{1}{T}\sum_{t=0}^{T-1}\mathbb{E}\left[V^{\pi_t}(\mu)\right]\\
	\leq\;& \frac{\log(e|\mathcal{A}|)}{(1-\gamma)^2\beta T}+\frac{4}{(1-\gamma)^2T}\sum_{t=0}^{T-1}\mathbb{E}[\|Q^{\pi_t}-Q_{t+1}\|_\infty]\\
	\leq \;&\frac{\log(e|\mathcal{A}|)}{(1-\gamma)^2\beta T}+\frac{4\max(0,(1-\bar{\rho}\min_{s,a}\pi_b(a|s)))}{(1-\gamma)^4}+\frac{24}{(1-\gamma)^3}\left(1-\frac{1-\gamma_c}{2}\alpha\right)^{\frac{1}{2}[K-(\tau_\alpha+n+1)]}\\
	&+\frac{1200\log^{1/2}(|\mathcal{S}||\mathcal{A}|)}{(1-\gamma)^3(1-\gamma_c)}(\bar{\rho}+1)f(\bar{c},\gamma)[\alpha (\tau_\alpha+n+1)]^{1/2}.
\end{align*}
This proves Theorem \ref{thm:main}.

\subsection{Proof of Corollary \ref{co:sample_complexity}}\label{pf:co:sample_complexity}

We begin with the result of Theorem \ref{thm:main} when $\bar{\rho}=\frac{1}{\min_{s,a}\pi_b(a|s)}$ (which ensures $E_3=0$):
\begin{align*}
    V^{\pi^*}(\mu)-\max_{0\leq t\leq T-1}\mathbb{E}\left[V^{\pi_{t}}(\mu)\right]
    \leq \;&\underbrace{\frac{24}{(1-\gamma)^3}\left(1-\frac{M_{\min}(1-\gamma)(1-(\gamma C_{\min})^n)}{2(1-\gamma C_{\min})}\alpha\right)^{\frac{1}{2}(K- (\tau_\alpha+n+1))}}_{E_1}\\
    &+\underbrace{\frac{1200(1-\gamma C_{\min})\log^{1/2}(|\mathcal{S}||\mathcal{A}|)}{M_{\min}(1-\gamma)^4(1-(\gamma C_{\min})^n)}(\bar{\rho}+1)f(\bar{c},\gamma)[\alpha (\tau_\alpha+n+1)]^{1/2}}_{E_2}\\
    &+\underbrace{\frac{\log(e|\mathcal{A}|)}{(1-\gamma)^2\beta T}}_{E_4},
\end{align*}
where we used the explicit expression of $\gamma_c$ in Proposition \ref{prop:Q-trace-property} (3) (b).
Our goal is to obtain an $\epsilon$-optimal policy, i.e., $V^{\pi^*}(\mu)-\max_{0\leq t\leq T-1}\mathbb{E}\left[V^{\pi_{t}}(\mu)\right]\leq \epsilon$.

We begin with the term $E_4$. It is clear that in order for $E_4\leq \epsilon$, we need to have $T=\mathcal{O}(\epsilon^{-1}(1-\gamma)^{-2})$. Now consider the term $E_2$. Since $\tau_\alpha \leq L(\log(1/\alpha)+1)$ for some $L>0$, the inequality $E_2\leq \epsilon$ implies 
\begin{align*}
    \alpha\sim \Tilde{\mathcal{O}}\left(\frac{(1-\gamma)^8 M_{\min}^2}{\bar{\rho}^2}\right)\mathcal{O}\left(\frac{\epsilon^2}{\log(1/\epsilon)}\right)
\end{align*}
Finally, using $\alpha$ in the term $E_1$ and the inequality that $e^x\geq 1+x$ for all $x\in\mathbb{R}$, then we have $E_1\leq \epsilon$ when 
\begin{align*}
    K=\mathcal{O}(\epsilon^{-2}\log^2(1/\epsilon)) \Tilde{\mathcal{O}}((1-\gamma)^{-9}M_{\min}^{-3}\bar{\rho}^2)
\end{align*}
It follows that the sample complexity is
\begin{align*}
    TK=\mathcal{O}(\epsilon^{-3}\log^2(1/\epsilon)) \Tilde{\mathcal{O}}((1-\gamma)^{-11}M_{\min}^{-3}\bar{\rho}^2)
\end{align*}
To determine the dependence on the size of the state-action space, observe that
\begin{align*}
    M_{\min}=\min_{s,a}\mu_b(s)\pi_b(a|s)\leq  \frac{1}{|\mathcal{S}||\mathcal{A}|} \quad \text{and} \quad  \bar{\rho}=\frac{1}{\min_{s,a}\pi_b(a|s)}\geq |\mathcal{A}|,
\end{align*}
where the equalities are attained when $\mu_b(s)=\frac{1}{|\mathcal{S}|}$ for all $s$ and $\pi_b(a|s)=\frac{1}{|\mathcal{A}|}$ for all $a$. Therefore, we have at least $\mathcal{O}(|\mathcal{S}|^3|\mathcal{A}|^5)$ dependence on the state and action space.

\subsection{Proof of All Technical Lemmas}

\subsubsection{Proof of Lemma \ref{le:logzt}}\label{pf:le:logzt}
Using the update rule of $\pi_{t}$ in Algorithm \ref{alg:off_policy_NPG} and we have for any $t\geq 0$ and $s\in\mathcal{S}$:
\begin{align*}
	\log(Z_{t}(s))&=\log\left[\sum_{a\in\mathcal{A}}\pi_t(a|s)\exp(\beta (Q_{t+1}(s,a)-V^{\pi_t}(s))\right]\\
	&\geq \beta\sum_{a\in\mathcal{A}}\pi_t(a|s)(Q_{t+1}(s,a)-V^{\pi_t}(s))\tag{Jensen's inequality}\\
	%&= \beta\sum_{a\in\mathcal{A}}\pi_t(a|s)(Q_{t+1}(s,a)-Q^{\pi_t}(s,a)+Q^{\pi_t}(s,a)-V^{\pi_t}(s))\\
	&=\beta\sum_{a\in\mathcal{A}}\pi_t(a|s)(Q_{t+1}(s,a)-Q^{\pi_t}(s,a)).
\end{align*}

\subsubsection{Proof of Lemma \ref{le:Vmu}}\label{pf:le:Vmu}
For any starting distribution $\mu$, we have
\begin{align*}
	V^{\pi_{t+1}}(\mu)-V^{\pi_t}(\mu)=\;&\frac{1}{1-\gamma}\mathbb{E}_{s\sim d^{t+1}}\sum_{a\in\mathcal{A}}\pi_{t+1}(a|s)A^{\pi_t}(s,a)\tag{Performance Difference Lemma}\\
	=\;&\frac{1}{1-\gamma}\mathbb{E}_{s\sim d^{t+1}}\sum_{a\in\mathcal{A}}\pi_{t+1}(a|s)(Q^{\pi_t}(s,a)-Q_{t+1}(s,a)+Q_{t+1}(s,a)-V^{\pi_t}(s))\\
	=\;&\frac{1}{1-\gamma}\mathbb{E}_{s\sim d^{t+1}}\sum_{a\in\mathcal{A}}\pi_{t+1}(a|s)(Q^{\pi_t}(s,a)-Q_{t+1}(s,a))\\
	&+\frac{1}{(1-\gamma)\beta}\mathbb{E}_{s\sim d^{t+1}}\sum_{a\in\mathcal{A}}\pi_{t+1}(a|s)\log\left(\frac{\pi_{t+1}(a|s)}{\pi_t(a|s)}Z_{t}(s)\right)\tag{Algorithm \ref{alg:off_policy_NPG}}\\
	=\;&\frac{1}{1-\gamma}\mathbb{E}_{s\sim d^{t+1}}\sum_{a\in\mathcal{A}}\pi_{t+1}(a|s)(Q^{\pi_t}(s,a)-Q_{t+1}(s,a))\\
	&+\frac{1}{(1-\gamma)\beta}\mathbb{E}_{s\sim d^{t+1}}\mathcal{KL}(\pi_{t+1}(\cdot|s)\mid \pi_t(\cdot|s))+\frac{1}{(1-\gamma)\beta}\mathbb{E}_{s\sim d^{t+1}}\log Z_{t}(s)\\
	\geq \;&\frac{1}{1-\gamma}\mathbb{E}_{s\sim d^{t+1}}\sum_{a\in\mathcal{A}}\pi_{t+1}(a|s)(Q^{\pi_t}(s,a)-Q_{t+1}(s,a))+\frac{1}{(1-\gamma)\beta}\mathbb{E}_{s\sim d^{t+1}}\log Z_{t}(s)\\
	=\;&\frac{1}{1-\gamma}\mathbb{E}_{s\sim d^{t+1}}\sum_{a\in\mathcal{A}}\pi_{t+1}(a|s)(Q^{\pi_t}(s,a)-Q_{t+1}(s,a))\\
	&+\frac{1}{(1-\gamma)\beta}\mathbb{E}_{s\sim d^{t+1}}\left[\log Z_{t}(s)-\beta\sum_{a\in\mathcal{A}}\pi_t(a|s)(Q_{t+1}(s,a)-Q^{\pi_t}(s,a))\right]\\
	&+\frac{1}{1-\gamma}\mathbb{E}_{s\sim d^{t+1}}\sum_{a\in\mathcal{A}}\pi_t(a|s)(Q_{t+1}(s,a)-Q^{\pi_t}(s,a))\\
	\geq \;&\frac{1}{1-\gamma}\mathbb{E}_{s\sim d^{t+1}}\sum_{a\in\mathcal{A}}(\pi_t(a|s)-\pi_{t+1}(a|s))(Q_{t+1}(s,a)-Q^{\pi_t}(s,a))\\
	&+\frac{1}{\beta}\mathbb{E}_{s\sim \mu}\left[\log Z_{t}(s)-\beta\sum_{a\in\mathcal{A}}\pi_t(a|s)(Q_{t+1}(s,a)-Q^{\pi_t}(s,a))\right]\tag{$d^{t+1}\geq (1-\gamma)\mu$ and Lemma \ref{le:logzt}}\\
	= \;&\frac{1}{1-\gamma}\mathbb{E}_{s\sim d^{t+1}}\sum_{a\in\mathcal{A}}(\pi_t(a|s)-\pi_{t+1}(a|s))(Q_{t+1}(s,a)-Q^{\pi_t}(s,a))\\
	&-\mathbb{E}_{s\sim \mu}\sum_{a\in\mathcal{A}}\pi_t(a|s)(Q_{t+1}(s,a)-Q^{\pi_t}(s,a))+\frac{1}{\beta}\mathbb{E}_{s\sim \mu}\log Z_{t}(s).
\end{align*}

\subsubsection{Proof of Lemma \ref{le:Vnu}}\label{pf:le:Vnu}
Using the update rule of $\pi_{t}$ in Algorithm \ref{alg:off_policy_NPG} and we have for any $t\geq 0$ and $s\in\mathcal{S}$:
\begin{align*}
	V^{\pi^*}(\mu)-V^{\pi_t}(\mu)=\;&\frac{1}{1-\gamma}\mathbb{E}_{s\sim d^*}\sum_{a\in\mathcal{A}}\pi^*(a|s)A^{\pi_t}(s,a)\tag{Performance Difference Lemma}\\
	=\;&\frac{1}{1-\gamma}\mathbb{E}_{s\sim d^*}\sum_{a\in\mathcal{A}}\pi^*(a|s)(Q^{\pi_t}(s,a)-Q_{t+1}(s,a)+Q_{t+1}(s,a)-V^{\pi_t}(s))\\
	=\;&\frac{1}{1-\gamma}\mathbb{E}_{s\sim d^*}\sum_{a\in\mathcal{A}}\pi^*(a|s)(Q^{\pi_t}(s,a)-Q_{t+1}(s,a))\\
	&+\frac{1}{(1-\gamma)\beta}\mathbb{E}_{s\sim d^*}\sum_{a\in\mathcal{A}}\pi^*(a|s)\log\left(\frac{\pi_{t+1}(a|s)}{\pi_t(a|s)}Z_{t}(s)\right)\tag{Algorithm \ref{alg:off_policy_NPG}}\\
	=\;&\frac{1}{1-\gamma}\mathbb{E}_{s\sim d^*}\sum_{a\in\mathcal{A}}\pi^*(a|s)(Q^{\pi_t}(s,a)-Q_{t+1}(s,a))+\frac{1}{(1-\gamma)\beta}\mathbb{E}_{s\sim d^*}\log(Z_{t}(s))\\
	&+\frac{1}{(1-\gamma)\beta}\mathbb{E}_{s\sim d^*}\left[\mathcal{KL}(\pi^*(\cdot|s)\mid \pi_t(\cdot|s))-\mathcal{KL}(\pi^*(\cdot|s)\mid \pi_{t+1}(\cdot|s))\right]
\end{align*}

\section{Related Literature}

\subsection{Interpretation of Convergence Rate in terms of Sample Complexity}\label{ap:literature:convergence_bound}
Suppose we have a stochastic approximation algorithm that arises in RL, which has the following convergence bound:
\begin{align}\label{eq:100}
    \text{Error}\leq \frac{1}{T}+E_0,
\end{align}
where $T$ is the number of iterations,  and $E_0$ represents certain error that cannot be eliminated asymptotically. For example, when studying TD-learning with function approximation, $E_0$ represents the approximation error, i.e., the gap between the true value function and the best value function offered by the approximating function space.

\subsubsection{Global Convergence}

Consider the case where $E_0=0$. 
In this case, sample complexity is well-defined, and it stands for the number of samples required to make the appropriately  defined  error $\epsilon$.
In the TD-learning example, this corresponds to using a tabular representation. Specifically, in view of Eq. (\ref{eq:100}), the convergence rate is $\mathcal{O}(1/T)$. Moreover, suppose every iteration requires one sample. Then to obtain $\epsilon$ accuracy, the amount of sample required is  $\mathcal{O}(\epsilon^{-1})$.

\subsubsection{Convergence in the Presence of a Bias}\label{subsubsec:c_bias}

Now consider the case where $E_0\neq 0$. 
We argue that the definition of sample complexity in unclear, and needs careful consideration. 
In the TD-learning example, this corresponds to using function approximation, which induces an unbeatable error due to the limitation of the approximating function space. A similar situation arises in off-policy NAC algorithm studied in this paper if the IS ratios are truncated to a certain level.

Suppose we apply the AM-GM inequality $\frac{1}{N}\sum_{i=1}^N x_i\geq (\prod_{i=1}^Nx_i)^{1/N}$ ($x_i\geq 0$ for all $i$ and $N\in\mathbb{Z}^+$) to the RHS of Eq. (\ref{eq:100}). Then we obtain for any $a>0$ and $N\geq 1$:
\begin{align}
    \text{Error}&\leq \frac{1}{T}+E_0\nonumber\\
    &=\left(\frac{1}{T^Na^{N-1}}\times a^{N-1}\right)^{1/N}+E_0\nonumber\\
    &\leq \frac{1}{N}\left(\frac{1}{T^N a^{N-1}}+\underbrace{a+a+\cdots+a}_{N-1}\right)+E_0\nonumber\\
    &=\frac{1}{N a^{N-1}}\frac{1}{T^N}+\left(1-\frac{1}{N}\right)a+E_0.\label{eq:101}
\end{align}
Now if we choose $a=E_0$, then the previous inequality can be written as
\begin{align}
    \text{Error}&\leq \frac{1}{N E_0^{N-1}}\frac{1}{T^n}+\left(2-\frac{1}{N}\right)E_0\label{eq:103}\\ &=\mathcal{O}\left(\frac{1}{T^N}\right)+\mathcal{O}(E_0).\label{eq:102}
\end{align}
While the derivation in \eqref{eq:103} is correct, it leads to 
%The above inequality leads to
the following \textit{misleading} interpretation:

\textit{We have a sample complexity of $\mathcal{O}(\epsilon^{-1/N})$ with an asymptotic error of size $\mathcal{O}(E_0)$ for any $N\geq 1$}.

Clearly, this interpretation is incorrect.
By using the AM-GM trick, we obtained a better rate of convergence, but with a worse asymptotic error. Therefore, as long as one does not have global convergence (i.e., $E_0\neq 0$), it is not entirely clear how to define sample complexity. 
One possible way out of this confusion is to define sample complexity only when the asymptotic error is exactly $E_0$ (instead of the weakened $\mathcal{O}(E_0)$). An alternate way is to define sample complexity in terms of convergence to the exact solution of a \textit{modified} problem, and then separately characterize the error between the solution of the original problem and the modified problem. This was the approach taken in the classic paper on TD with linear function approximation \cite{tsitsiklis1997analysis}.

\subsubsection{The results in \cite{xu2020improving} and \cite{xu2020non}}

Convergence of AC type algorithms was studied in \cite{xu2020improving} and \cite{xu2020non} under linear function approximation. A special case of linear function approximation is the tabular setting, where the feature vectors are chosen to be the canonical basis vectors. However, in this case, the results in \cite{xu2020improving} and \cite{xu2020non} do not guarantee global convergence because of the presence of additive constants in the error (for any choice of $\lambda$ as defined in  \cite{xu2020improving,xu2020non}).

Now, consider the case of linear function approximation. We believe that the sample complexity of $\Tilde{\mathcal{O}}(\epsilon^{-2.5})$ for AC  and $\Tilde{\mathcal{O}}(\epsilon^{-4})$ sample complexity for NAC claimed in \cite{xu2020non}  and $\Tilde{\mathcal{O}}(\epsilon^{-2})$ sample complexity for NAC claimed in  \cite{xu2020improving} are misleading because they were essentially obtained in the manner described in Section \ref{subsubsec:c_bias}. We present more details below. 

However, if one agrees with the interpretation of the 
%{\color{red}If you believe in the 
sample complexity results in \cite{xu2020improving,xu2020non}, then our sample complexity results
can also be ``improved'' in the same sense as those in \cite{xu2020improving,xu2020non} to 
$\mathcal{O}(\epsilon^{-1/N})$ for any $N>0$ by doubling the asymptotic error from  $E_3$  to $2E_3$ (where the term $E_3$ is the truncation error given in Theorem \ref{thm:main}). This can be done by applying the AM-GM inequality described in Section \ref{subsubsec:c_bias} to the result of Theorem \ref{thm:main}. In fact, for any convergence bounds of the form (\ref{eq:100}) in the literature, one can use the same technique to obtain arbitrarily good convergence rate and sample complexity.

\paragraph{The paper \cite{xu2020improving}:}
Now we illustrate how \cite{xu2020improving} uses the AM-GM trick described above implicitly.
In Algorithm 1 (Actor-critic (AC) and natural actor-critic (NAC) online algorithms), the parameter $\lambda$ (line 19 of Algorithm 1) is introduced. In the statement of their main result (Eq. (31) of Theorem 6 in Appendix G), the parameter $\lambda$ appears both in the denominator of the $1/T$ terms (which is \textit{not} revealed in Theorem 3 of their main paper) and the numerator of the constant terms. We see that the role of the parameter $\lambda$ is essentially equivalent to the tunable constant $a$ introduced in Eq. (\ref{eq:101}) of our above derivation. Later the parameter $\lambda$ is set to be equal to $\sqrt{\zeta^{critic}_{approx}}$ (where $\zeta^{critic}_{approx}$ is the unbeatable error due to function approximation in TD-learning) so that the additive constant terms are absorbed into the $\mathcal{O}(\sqrt{\zeta^{critic}_{approx}})$ term, which eventually leads to their claim of obtaining $\Tilde{\mathcal{O}}(\epsilon^{-2})$ sample complexity of NAC. This is analogous to going from Eq. (\ref{eq:101}) to Eq. (\ref{eq:102}) of our derivation in Section \ref{subsubsec:c_bias} by setting $a=E_0$.

\paragraph{The paper \cite{xu2020non}:}
Now we illustrate how the proof in  \cite{xu2020non} is essentially equivalent to using the  AM-GM trick described above.
In Algorithm 1 (Two Time-scale AC and NAC), the parameter $\lambda$ is introduced to perform the critic update. Later in the third bullet point on the same page of Algorithm 1, the parameter $R_\theta$ is set to be $\mathcal{O}(\lambda^{-1})$.

Consider the resulting bounds in all 5 cases in step 3 of the proof of Theorem 2 (Appendix C). The parameter $\lambda$ appears in the numerator of the constant term while the parameter $R_\theta=\mathcal{O}(\lambda^{-1})$ appears quadratically in the $1/t^{1-\sigma}$ term (which is \textit{not} revealed in the statement of Theorem 2 in the main paper). This leads to the claim of $\mathcal{O}(\epsilon^{-2.5})$ sample complexity for AC with asymptotic error $\mathcal{O}(\lambda)$. We believe this is analogous to Eq. (\ref{eq:101}) of our above derivation. 

Consider the resulting bounds in all 5 cases in step 2 of the proof of Theorem 3 (Appendix D). The parameter $\lambda$ appears in the numerator of the constant term while the parameter $R_\theta=\mathcal{O}(\lambda^{-1})$ appears quadratically in the $1/t^{1-\sigma}$ term (which is \textit{not} revealed in the statement of Theorem 3 in the main paper). In Theorem 3, the parameter $\lambda$ is set to be $\mathcal{O}(\sqrt{\xi_{approx}'})$, which eventually leads the claim of $\mathcal{O}(\epsilon^{-4})$ sample complexity for NAC with asymptotic accuracy $\mathcal{O}(\sqrt{\xi_{approx}'})$. We believe this is analogous to Eqs. (\ref{eq:103}) and (\ref{eq:102}) of our above derivation.

\subsection{Single Trajectory}\label{ap:literature:trajectory}
The AC and NAC algorithms presented in \cite{xu2020non} and \cite{xu2020improving}  appear to be based on a single trajectory, at first glance. The single sample path is not from the original transition matrix $P$, but is from a modified transition matrix, $\Tilde{P}(\cdot\mid s,a)=\gamma P(\cdot\mid s,a)+(1-\gamma) \xi(\cdot)$, where $\xi(\cdot)$ is the initial distribution. Now, in order to sample from the modified transition matrix $\Tilde{P}$, one has to sample from the original matrix $P$ with probability $\gamma$ and  reset to a state sampled from $\xi(\cdot)$  with probability $(1-\gamma)$. Thus, in reality, the algorithms in  \cite{xu2020non,xu2020improving} are \textit{not} based on a single trajectory. This is made explicitly clear in \cite{wang2019neural} (Section 3.2.1. Actor Update: Sampling From Visitation Measure), where the same modified transition matrix $\Tilde{P}$ was used.

\subsection{The Issue of Exploration}\label{ap:literature:exploration}

A major issue with on-policy AC and NAC is exploration. In related literature, to establish convergence bounds, usually it requires either hard-to-satisfy assumptions to ensure exploration, or additional exploration steps which slow down the convergence rate.

\subsubsection{Hard-to-Satisfy Assumptions}
The convergence of AC type methods have been established in several previous work. Each of these results require some regularity assumptions on the underlying system. However, in the simple tabular setting, one can show that these assumptions fail to hold. In particular, Assumption 4.1 in \cite{wu2020finite}, Assumption 1 in \cite{xu2020improving, xu2020non, kumar2019sample}, and Assumption 3.3 in \cite{qiu2019finite} in the tabular setting imply the sequence of policies $\{\pi_t\}$ generated by the algorithm satisfy $\pi_t(a|s) \geq \delta > 0$, for all $s,a$ and $t$. This assumption in conjunction with the irreducibility assumption of the underlying Markov chain under all the policies, one can show that all the state and actions will be visited infinitely often as the AC algorithm proceeds. 

The above mentioned assumption means that all the elements of the policy table must attain at least positive value $\delta$ uniformly over time. However, a well known result shows that, for every MDP there always exist an optimal deterministic policy \cite{puterman1995markov}. 
In particular, one can construct MDPs with a unique deterministic optimal policy. In such examples, some of the elements of $\pi_t$ should converge to zero as the AC algorithm proceeds, and this violates the aforementioned assumptions. For more information, look at Section 4 in \cite{khodadadian2021finite} where an experimental implementation of NAC shows that $\pi_t$ indeed converges to a deterministic policy.

\subsubsection{Additional Exploration Steps}
One way of avoiding the assumption mentioned in the previous subsection is to artificially introduce additional exploration. This was done in \cite{khodadadian2021finite} where  $\epsilon$-greedy NAC was proposed under which, at each time, actions are sampled from $\epsilon$-greedy policy $\hat{\pi}_t = (1-\epsilon_t)\pi_t + \frac{\epsilon_t}{|\mathcal{A}|}$. Sampling from this policy ensures that all actions will be visited with probability at least $\frac{\epsilon_t}{|\mathcal{A}|}$, which ensures exploration of all state-action pairs. However, this sampling policy will result in a slower rate of convergence as stated in \cite{khodadadian2021finite}.

\subsection{Sample Complexity Calculation in Related Literature}\label{ap:compute-sample-complexity}

In this section, we compute the sample complexity of each related work listed in Table \ref{table: results}, based on the convergence bounds provided in the corresponding paper. We will use the same notation as was used in the corresponding paper.
\subsubsection{\cite{wang2019neural}}

\paragraph{AC, Theorem 4.7:}

In order the obtain an $\epsilon$-optimal stationary point, we need $T=\mathcal{O}(\epsilon^{-2/3})$, which implies $m=\mathcal{O}(\epsilon^{-16/3}
)$. Since $T_{TD}=\Omega(m)$, the total sample complexity $T\times T_{TD}$ is at least $\mathcal{O}(\epsilon^{-6})$.

\paragraph{NAC, Corollary 4.14:}

In order to obtain an $\epsilon$-optimal policy, we need $T=\mathcal{O}(\epsilon^{-2})$, which implies $m=\Omega(\epsilon^{-14})$. Since $T_{TD}=\Omega(m)$, the total sample complexity $T\times T_{TD}$ is at least $\mathcal{O}(\epsilon^{-14})$.

\subsubsection{\cite{kumar2019sample}}
\textbf{AC, Theorem 1:} 
The result in this paper assumes a convergence rate of $\mathcal{O}(1/k^b)$ for the critic. It was shown in \cite{srikant2019finite} that a rate of $\mathcal{O}(1/\sqrt{k})$ is achievable, and so we use this to evaluate sample complexity. 
%According to Theorem 1  in this paper, suppose $b=1/2$ (This can be achieved by results from \cite{chen2021finite}).
In order to obtain $\epsilon$-close stationary point, we need $\mathcal{O}(1+2+\dots+\epsilon^{-2}) = \mathcal{O}(\epsilon^{-4})$ number of samples, which implies $\mathcal{O}(\epsilon^{-4})$ sample complexity.

\subsubsection{\cite{agarwal2019theory}}
\textbf{NAC, Corollary 6.2:} 
In order to obtain an $\epsilon$-optimal policy, we need to have $\frac{1}{1-\gamma}\frac{1}{\sqrt{T}}\leq \mathcal{O}(\epsilon)$ and $\frac{1}{(1-\gamma)^2}\frac{1}{N^{1/4}}\leq \mathcal{O}(\epsilon)$. This is equivalent to $T\geq \mathcal{O}((1-\gamma)^{-2}\epsilon^{-2})$ and $N\geq \mathcal{O}((1-\gamma)^{-8}\epsilon^{-4})$. Hence, the total sample complexity is $\frac{2TN}{1-\gamma}=\mathcal{O}((1-\gamma)^{-11}\epsilon^{-6})$. Note that although \citep[Corollary 6.2]{agarwal2019theory} is stated for the function approximation setting, the result would be the same even in the tabular setting.

\subsubsection{\cite{khodadadian2021finite}}
\textbf{NAC, Corollary 1.1:} As stated in Corollary 1.1 of this paper, to obtain an $\epsilon$ optimal policy, we have $T\geq \mathcal{O}(\epsilon^{-4})$ sample complexity. 

\subsubsection{\cite{qiu2019finite}}
\textbf{AC, Theorem 4.6:} As stated in Theorem 4.6 of this paper, in order to an $\epsilon$-optimal stationary point, we need $T \geq  \epsilon^{-2}$ number of outer loops, and in each outer loop we need $\mathcal{O}(T)$ inner loops. Hence, the total sample complexity is $T\times T = \mathcal{O}(\epsilon^{-4})$. 

\section{Experimental Results} \label{sec:exp_result_details}
\subsection{Details of the Experimental Results} \label{sec:D.1}
Figure \ref{fig:1} shows the convergence behavior of off-policy NAC \ref{alg:off_policy_NPG}. The underlying process is a MDP with 5 states and 3 actions $\{a_1, a_2, a_3\}$ and $\gamma = 0.9$. The state transition probabilities over the states are
\begin{align*}
    P_{a_1}=\begin{bmatrix}
0 & 1 & 0 & 0 & 0\\
0 & 0 & 1 & 0 & 0\\
0 & 0 & 0 & 1 & 0\\
0 & 0 & 0 & 0 & 1\\
1 & 0 & 0 & 0 & 0
\end{bmatrix}, ~~~~~~ P_{a_2}=\begin{bmatrix}
1 & 0 & 0 & 0 & 0\\
0 & 1 & 0 & 0 & 0\\
0 & 0 & 1 & 0 & 0\\
0 & 0 & 0 & 1 & 0\\
0 & 0 & 0 & 0 & 1
\end{bmatrix}, ~~~~~~ P_{a_3}=\begin{bmatrix}
0 & 0 & 0 & 0 & 1\\
1 & 0 & 0 & 0 & 0\\
0 & 1 & 0 & 0 & 0\\
0 & 0 & 1 & 0 & 0\\
0 & 0 & 0 & 1 & 0
\end{bmatrix},
\end{align*}
and the reward functions are $\mathcal{R}(s,a_1) = 1$, $\mathcal{R}(s,a_2) = 0.5$, and $\mathcal{R}(s,a_3) = 0$ for all $s\in\mathcal{S}$. In this setting, clearly the optimal policy is to take action $a_1$ in all states. In addition, the behavior policy has uniform distribution, i.e. $\pi_b(a|s) = 1/3~ \forall a,s$. The parametters of the algorithm are chosen as follows: $n=6, T=100, K=1000, \alpha=0.05, \beta = 0.1, \bar{\rho}=3, \bar{c}=1, \pi_0(a|s)=\frac{1}{3} \forall a,s$. In addition, for $Q_0$ input of the $Q$-trace, we use the previously learned $Q$ table as the input to enhance the convergence. An implementation of the code is available at \href{https://github.com/gt-coar/off_policy-NAC/blob/main/off-policy_NAC.py}{https://github.com/gt-coar/off\_policy-NAC/blob/main/off-policy\_NAC.py}. It is clear that algorithm converges in Figure \ref{fig:1}.

\subsection{Using Repeated Samples in the Critic}
Using the same setup as above, we executed Algorithm \ref{alg:off_policy_NPG} with repeated samples used for each iteration of the $Q$-trace. Figure \ref{fig:2} shows the result. It is clear that we do not have convergence in this case.

\begin{figure}[ht]
    \centering
    \includegraphics[width=0.5\linewidth]{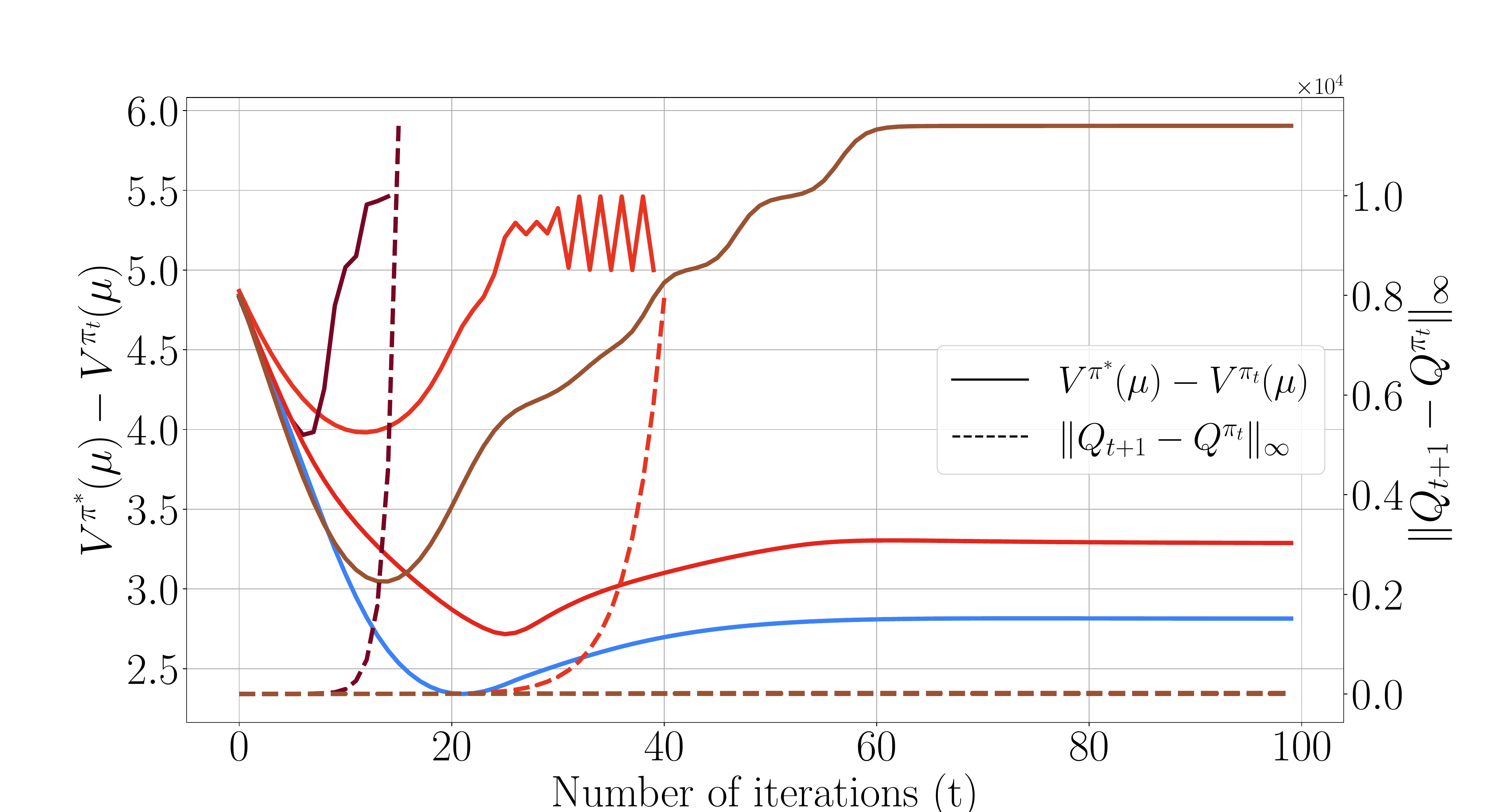}
    \caption{behavior of off-policy NAC when the critic updates are performed using a fixed number of samples. The straight lines are  $V^{\pi^*}(\mu)-V^{\pi_t}(\mu)$ values for 5 different sample paths, and the dashed lines are the corresponding critic errors $\|Q^{t+1} - Q^{\pi_t}\|_\infty$ of each sample path. It is clear that the algorithm does not converge.}
    \label{fig:2}
\end{figure}

\subsection{The Effect of the Truncation Levels}
In order to evaluate the effect of the truncation of the importance sampling in the behavior of the off-policy NAC, we run Algorithm \ref{alg:off_policy_NPG} for different levels of $\bar{\rho}$ and $\bar{c}$ for an MDP with the same setting as in section \ref{sec:D.1}. The result is shown in Figure \ref{fig:3}. In this figure, for each choice of the $\bar{\rho}$ and $\bar{c}$ we run the Algorithm \ref{alg:off_policy_NPG} for 6 number of times. The dashed lines represent the average of these 6 sample paths, and the area around the dashed lines represent the standard deviation of these 6 trajectories. It is clear that the choice of $\bar{\rho}=3, \bar{c}=1$ results in the best convergence with the lowest standard deviation. Reducing $\bar{\rho}=3$ to $\bar{\rho}=2.5$ is worsening the convergence bahaviour, and further increasing $\bar{c}=1$ to $\bar{c}=1.5$ increases the standard deviation.

\begin{figure}[h]
    \centering
    \includegraphics[width=0.4\linewidth]{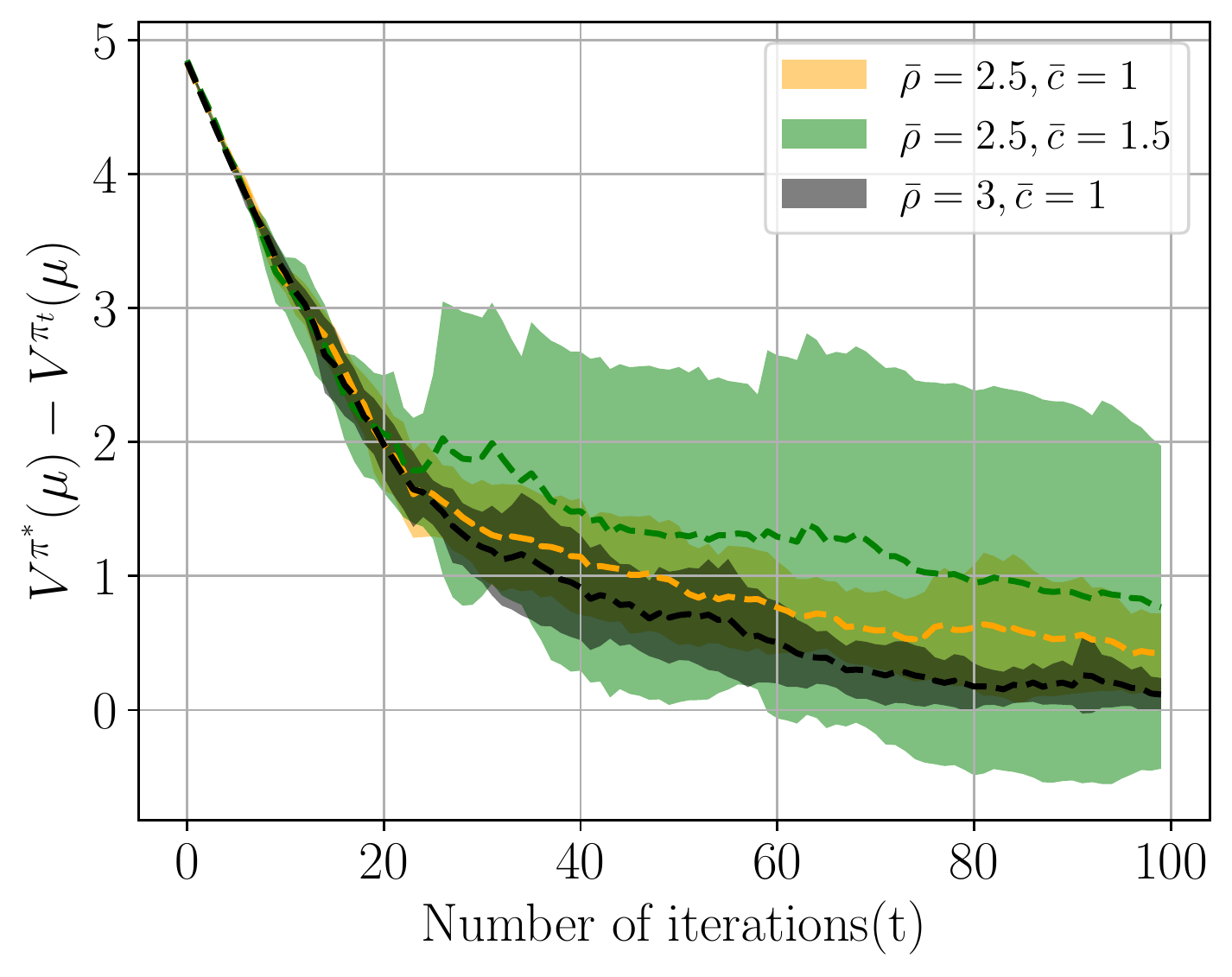}
    \caption{The convergence behavior of off-policy NAC with different levels of truncation. For each choice of the truncation level, we run Algorithm \ref{alg:off_policy_NPG} 6 times, and we plot the mean with the dashed line, and the standard deviation with the colored area. }
    \label{fig:3}
\end{figure}
\end{appendix}

\end{document}